\pdfoutput=1
\documentclass[pdflatex,sn-mathphys-num]{sn-jnl}
\usepackage{graphicx}%
\usepackage{multirow}%
\usepackage{amsmath,amssymb,amsfonts}%
\usepackage{amsthm}%
\usepackage{mathrsfs}%
\usepackage[title]{appendix}%
\usepackage{xcolor}%
\usepackage{textcomp}%
\usepackage{manyfoot}%
\usepackage{booktabs}%
\usepackage{algorithm}%
\usepackage{algorithmicx}%
\usepackage{algpseudocode}%
\usepackage{listings}%
\usepackage{amsthm}
\usepackage{xr}

\makeatletter
\usepackage{amsmath,amsfonts,bm}

\def\rd{{\textnormal{d}}}

\def\vw{{\bm{w}}}
\def\vx{{\bm{x}}}

\def\evb{{b}}

\def\evd{{d}}

\def\evw{{w}}
\def\evx{{x}}

\def\mA{{\bm{A}}}
\def\mB{{\bm{B}}}
\def\mC{{\bm{C}}}

\DeclareMathAlphabet{\mathsfit}{\encodingdefault}{\sfdefault}{m}{sl}
\SetMathAlphabet{\mathsfit}{bold}{\encodingdefault}{\sfdefault}{bx}{n}

\newtheorem{theorem}{Theorem}

\newtheorem{lemma}{Lemma}

\usepackage{tabularray}
\usepackage{graphicx}
\usepackage{subcaption}

\usepackage{algorithm}
\theoremstyle{thmstyleone}%

\usepackage{tikz}

\theoremstyle{thmstyletwo}%

\theoremstyle{thmstylethree}%

\begin{document}

\title[Article Title]{Mitigating Communication Costs: The Role of Dendritic Nonlinearity}

\author*[1,2,6]{\fnm{Xundong} \sur{Wu}}\email{wuxundong@gmail.com}
\equalcont{These authors contributed equally to this work.}

\author[2,3]{\fnm{Pengfei} \sur{Zhao}}
\equalcont{These authors contributed equally to this work.}

\author[1,2]{\fnm{Zilin} \sur{Yu}}
\equalcont{These authors contributed equally to this work.}

\author[2,5]{\fnm{Lei} \sur{Ma}}
\equalcont{These authors contributed equally to this work.}

\author[1]{\fnm{Yifan} \sur{Gao}}
\author[1]{\fnm{Ka-Wa} \sur{Yip}}

\author[6]{\fnm{Huajin} \sur{Tang}}

\author[6]{\fnm{Gang} \sur{Pan}}

\author[7]{\fnm{Poirazi} \sur{Panayiota}}

\author[2,4]{\fnm{Tiejun} \sur{Huang}}

\affil*[1]{\orgname{Zhejiang Lab, China},  \city{Hangzhou}, \state{Zhejiang}, \country{China}}

\affil[2]{\orgname{Beijing Academy of Artificial Intelligence}, \orgaddress{\city{Beijing}, \country{China}}}

\affil[3]{\orgname{Bytedance}, \orgaddress{\city{Beijing}, \country{China}}}

\affil[4]{\orgname{National Key Laboratory for Multimedia Information Processing, School of Computer Science, Peking University}, \orgaddress{\city{Beijing}, \country{China}}}

\affil[5]{\orgname{National Biomedical Imaging Center, Peking University}, \orgaddress{\city{Beijing}, \country{China}}}

\affil[6]{\orgname{Zhejiang University}, \orgaddress{\city{Hangzhou}, \country{China}}}
\affil[7]{\orgname{IMBB-FORTH}, \orgaddress{\city{Heraklion}, \state{Crete}, \country{Greece}}}

\abstract{Our understanding of biological neuronal networks has profoundly influenced the development of artificial neural networks (ANNs). However, neurons utilized in ANNs differ considerably from their biological counterparts, primarily due to the absence of complex dendritic trees with local nonlinearities. Early studies have suggested that dendritic nonlinearities could substantially improve the learning capabilities of neural network models. In this study, we systematically examined the role of nonlinear dendrites within neural networks. Utilizing machine-learning methodologies, we assessed how dendritic nonlinearities influence neural network performance. Our findings demonstrate that dendritic nonlinearities do not substantially affect learning capacity; rather, their primary benefit lies in enabling network capacity expansion while minimizing communication costs through effective localized feature aggregation. This research provides critical insights with significant implications for designing future neural network accelerators aimed at reducing communication overhead during neural network training and inference.
}


\maketitle

Over the past decade, artificial neural networks (ANNs) have significantly advanced across diverse domains, demonstrating impressive performance on complex tasks~\cite{silver_mastering_2017, openai2023gpt4}. Despite their inspiration from biological neuronal networks, modern ANNs use highly simplified "point neurons," described mathematically as: 
\begin{equation} h=\sigma\left(\sum_{i=1}^{n}{\evw_i\evx_i+b}\right), \label{equ:neuron} \end{equation} 
where inputs ($\evx_i$), weights ($\evw_i$), bias ($b$), and nonlinear activation ($\sigma$) produce neuron output ($h$). These neurons differ drastically from biological neurons, which have elaborate dendritic structures (Fig.~\ref{fig:neuron}).

\begin{figure}
    \centering
    \includegraphics[width=0.9\linewidth]{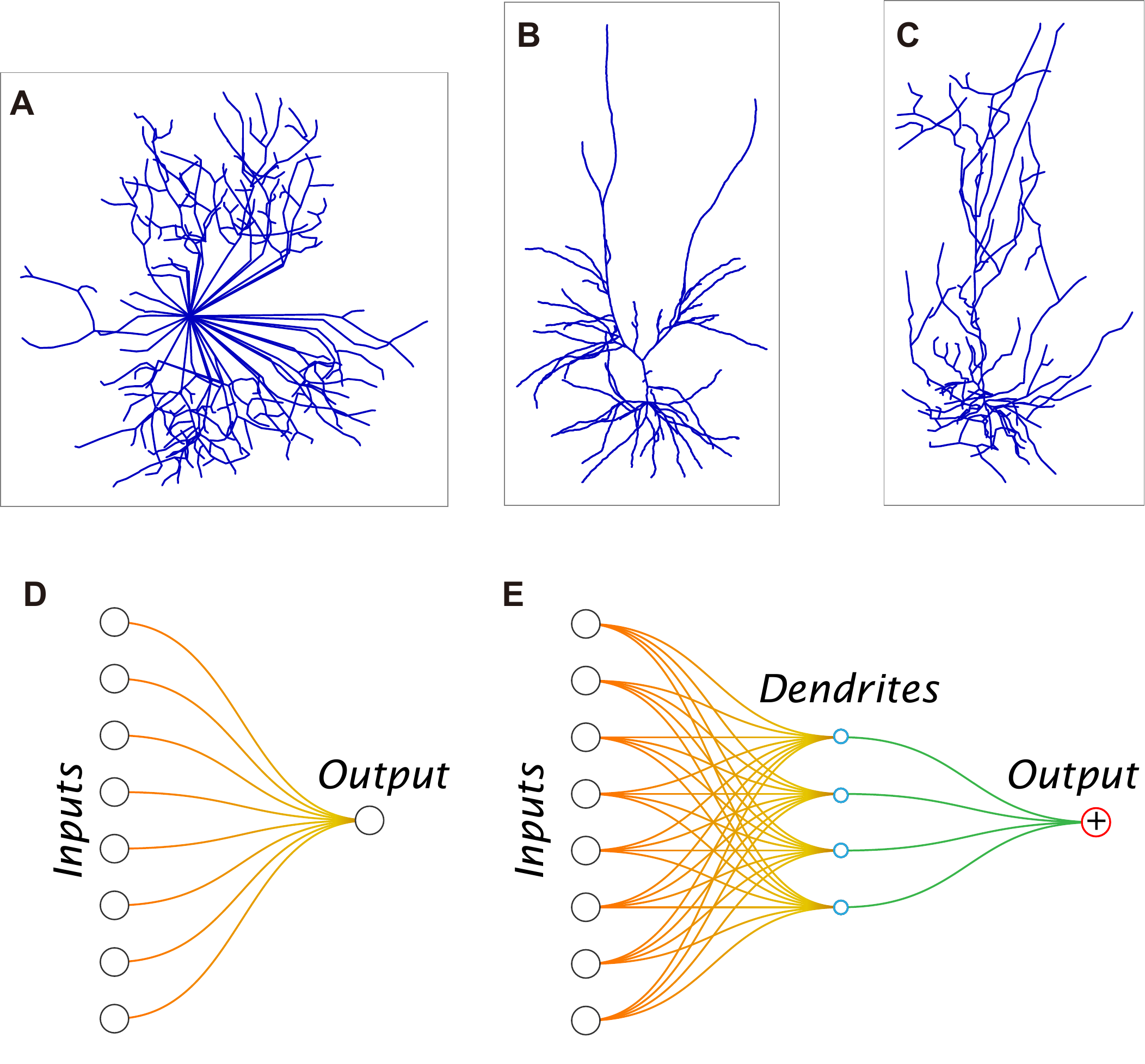}
    \caption[]{(\textbf{A, B, C}) Illustration of three representative neurons showcasing distinct dendritic structures from left to right: A chicken bipolar neuron~\cite{wang2012vivo}, a human hippocampal pyramidal neuron~\cite{benavides2020differential}, and a ferret neocortical pyramidal neuron~\cite{adusei2021morphological}. All neuronal morphologies are from the Neuromorpho.org database~\cite{ascoli2007neuromorpho}. (\textbf{D}) Portrays a point neuron, as characterized by Equation~\ref{equ:neuron}. (\textbf{E}) Illustrates a dendritic neuron with $4$ dendritic branches as detailed by Equations ~\ref{equ:den} and ~\ref{equ:dendrite}.}
    \label{fig:neuron}
\end{figure}

Dendritic structures in biological neurons offer enhanced surface-area-to-volume ratios, essential for forming numerous synaptic connections within limited brain space~\cite{stuart2016dendrites,chklovskii2000optimal,chklovskii2004synaptic}. Unlike biological brains, ANNs executed on general-purpose hardware (CPUs, GPUs) do not have physical constraints as in biological brains, raising questions about the necessity of dendrites in artificial systems. Here, we argue that dendrites are indeed valuable beyond biological contexts.

Dendrites are known to facilitate local nonlinear operations due to their anatomical compartmentalization and specialized voltage-gated ion channels~\cite{poirazi_pyramidal_2003, stuart2016dendrites,magee2000dendritic,major2013active}. These localized nonlinearities may enable key computational functions like coincidence detection, signal amplification, learning, and temporal discrimination~\cite{poirazi_impact_2001,jones2021might,richards2019dendritic}. Most importantly, it has long been believed that dendritic nonlinearities can endow neurons with greater model capacity than point neuron-based models~\cite{wu2018improved,poirazi_impact_2001}. 

However, our findings challenge this assumption. Approaching the question from a machine learning perspective, we demonstrate that adopting neurons with active dendrites has little effect on model learning capacity. Instead, we show that adopting active dendrite can significantly reduce communication costs in artificial neural network (ANN) models. Specifically, dendritic architectures enable localized processing that reduces the number of transmitted features without degrading performance. This is particularly important because communication overhead—primarily from data movement—dominates energy consumption in ANNs~\cite{dally_model_2022}, echoing biological evidence that highlights the high cost of axonal transmission relative to local computation~\cite{levy2021communication}.

Our findings indicate that the dendritic neuron-based model' expanded learning capacity~\cite{poirazi_impact_2001,wu2018improved} likely arises primarily from the sparse structure employed in their models and redundancy avoidance attributed to the smaller unit size of dendrites. Thus, we posit primary advantage of active dendrite lies in mitigating communication costs through effective local feature aggregation.

Our investigation also finds that adopting a dendritic structure can significantly reduce memory access and occupancy during inference and training of ANN models. These results have important implications for the development of efficient ANN architectures and hardware for real-world applications.
\section{Results}\label{results}

In this study, we use a simplified dendritic neuron model~\cite{poirazi_pyramidal_2003,polsky2004computational}, depicted in Fig.\ref{fig:neuron}E and mathematically described by Eqs.\ref{equ:den} and~\ref{equ:dendrite}. In this architecture, incoming signals at each dendrite are integrated by computing the dendritic output $d_j$, which is obtained from the weight vector $\vw_j$, input activation $\vx$, and an optional bias $b_j$. The dendritic output is transformed by an element-wise nonlinear function $\sigma$ and subsequently summed to produce the somatic output $h$, conveyed to downstream recipients:

\begin{align}
    \hat{\evd}_j&={{\vw}_{j}^\top \vx+\evb_j}, \,\evd_j=\sigma(\hat{\evd}_j),\label{equ:den}\\
    h&=\sum_{j=1}^{K}{\evd_j} \,.
    \label{equ:dendrite}
\end{align}

Each dendritic unit described here has the same information-processing capacity as a point neuron, but differs in how its output is conveyed downstream. Unlike point neurons, whose outputs are independently transmitted to downstream neurons, dendrites share a common channel for output transmission, typically resulting in information loss. This property is formally detailed in Theorem~\ref{thm:1} (Appendix).

\subsection{Communication vs computing in neural networks}
In neural networks, the energy required for computation is substantial, but communication is the primary energy bottleneck in modern hardware~\cite{dally_model_2022}. Communication can incur orders of magnitude higher costs than computation; for instance, as reviewed by Dally et al.~\cite{dally_model_2022}, adding two 32-bit numbers may consume approximately 20 femtojoules (fJ), while fetching these numbers from memory can require about 1.3 nanojoules (nJ)—roughly 64,000 times more energy.

Biological brains also exhibit significant energy expenditure and structural investment related to communication, as indicated by extensive white matter volume~\cite{mota2019white} and high metabolic demands~\cite{levy2021communication,attwell2001energy,yu2018evaluating}. The evolutionary pressure on biological systems to minimize these costs suggests potential insights for artificial systems. Our study proposes that incorporating active dendrites into artificial neural networks can significantly mitigate communication-related energy expenses.

\subsection{Evaluating the communication efficiency of the dendritic structure}

\subsubsection{Developing the dendritic neuron model}

To investigate the role active dendrites can play in neural networks, we compare performance of models that are constituted with point neurons or dendritic neurons respectively. We substitute the point neurons in conventional neural network models with dendritic neurons. Each nonlinear summation unit—point neuron or active dendrite—receives at most one copy of a specific input from the previous network layer. Each dendrite within a neuron receives an equal number of dense connections; hence, a dendritic neuron with $K$ branches receives $K$ times more inputs than its point-neuron counterpart. Clearly those two models are of very different computing and parametric complexity. We need to compare models on an equal footing.

\begin{figure*}[t!]
\centering
\includegraphics[width=0.75\linewidth]{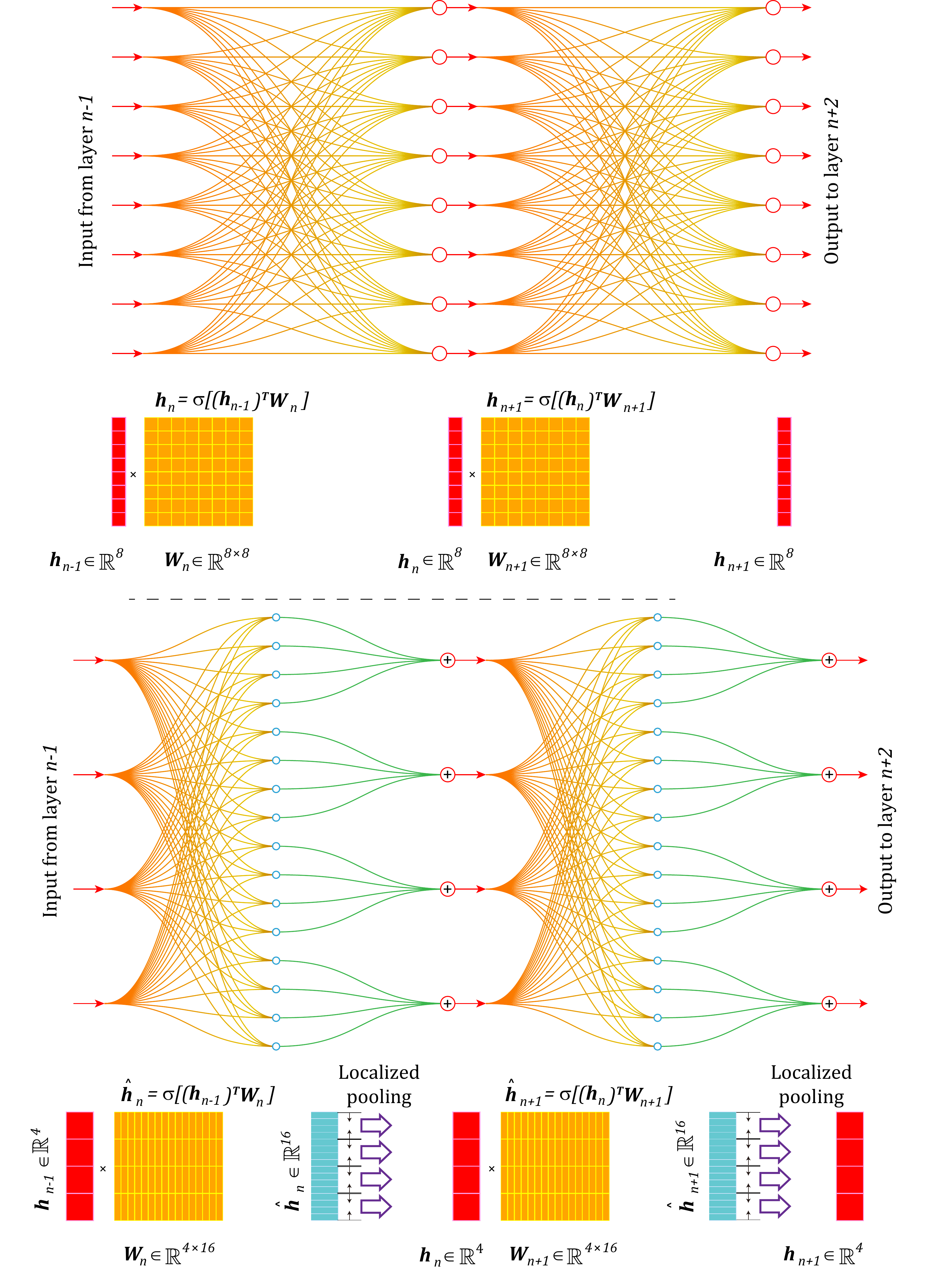}
\caption{Comparison of neural network layers using point neurons (\textbf{Top}) and dendritic neurons (\textbf{Bottom}). Only two layers from each model are depicted. The point neuron model has \(D=8\) channels, whereas the dendritic neuron model features neurons with $K=4$ dendritic branches each, leading to an effective \(\hat{D} = \frac{D}{\sqrt{4}} = 4\) channels. This ensures that both models have comparable parametric and computational complexities. Note: Tensor dimensions are symbolized by a mesh of patches; however, patch sizes do not reflect actual scale. Bias terms have been excluded for simplicity.}\label{fig:network}
\end{figure*}

Consider the example in Fig.~\ref{fig:network}: a fully connected layer with $D=8$ inputs and outputs (top) has a complexity of $D^2 = 64$. In contrast, a dendritic neuron layer (bottom) with $\hat{D}$ inputs/outputs neurons and $K=4$ dendrites per neuron has parametric complexity $K\hat{D}^2$. To equate this with the point-neuron layer, we set $\hat{D} = D/\sqrt{K}$, yielding $\hat{D} = 4$. (Since parametric and computational complexities always scale in the same way in this study, we will refer only to parametric complexity from here on.)

To compare communication costs, let $D$ be the total number of neurons in a layer or network, and define $\Psi$ as the ratio of $D$ relative to a point-neuron baseline. In the example, the point-neuron layer has $D = 8$, while the dendritic layer has $D = 4$, resulting in $\Psi = 0.5$.

\subsubsection{Dense models on ImageNet}

\begin{figure}[!ht]
    \centering
    \includegraphics[width=0.85\linewidth]{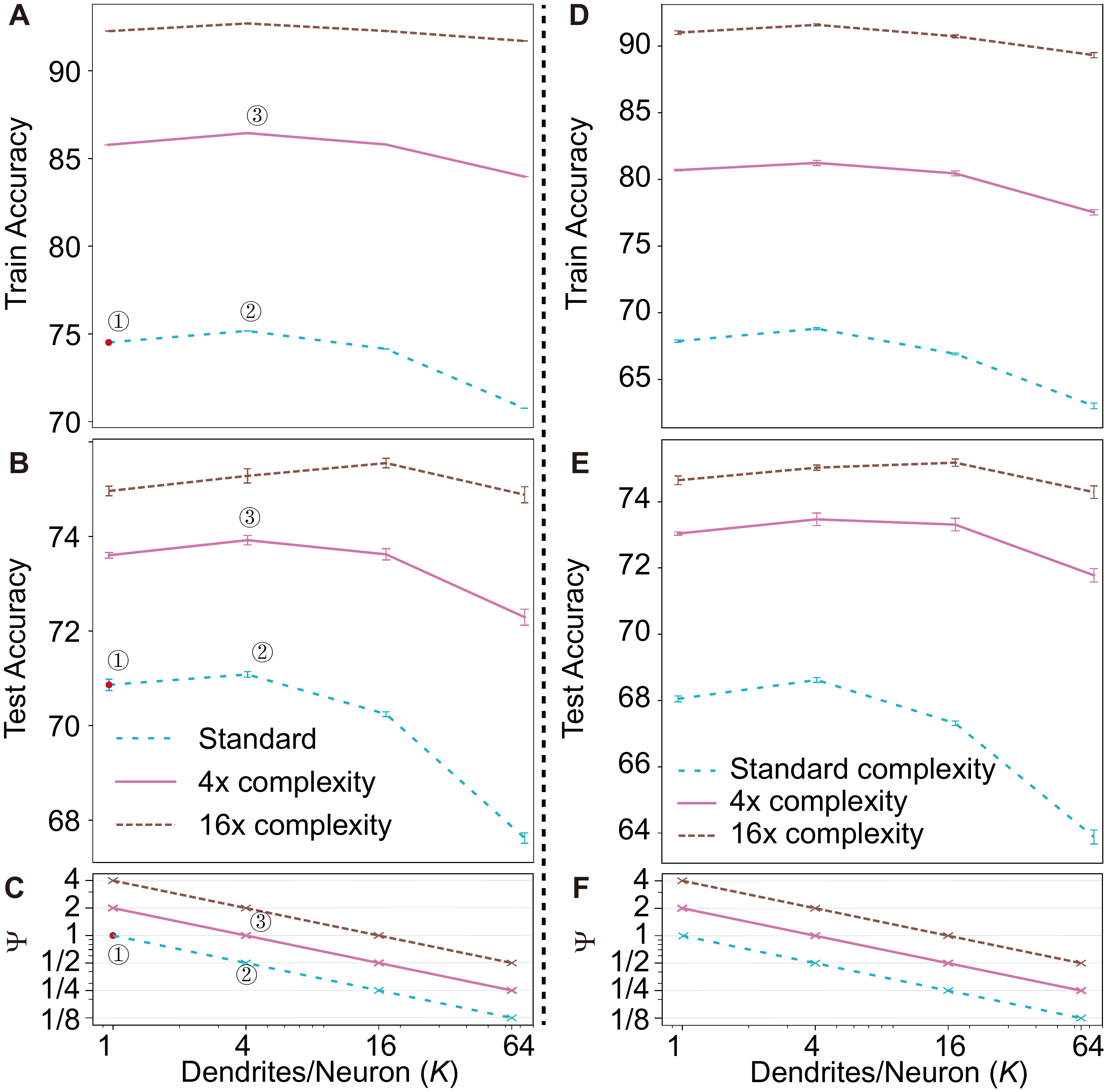}
    \centering
    \caption{Comparison of ResNet-18-style models using point vs. dendritic neurons on ImageNet. \textbf{Left (A-C):}  Experiment on dense models (5 trials, std dev shown). The red dot (\raisebox{.5pt}{\textcircled{\raisebox{-.9pt} {1}}}) marks the baseline ResNet-18 with standard point neurons (K=1). The x-axis indicates the number of dendrites per neuron (K).
    Three complexity levels are evaluated:
    Standard complexity (light blue dashed curves), models with the same complexity as baseline ResNet-18.
    4x complexity (solid magenta curves). 16x complexity (brown dashed curves). Within each curve, models share the same total parametric budget but differ in $K$ (number of dendrites per neuron). For example, model (\raisebox{.5pt}{\textcircled{\raisebox{-.9pt} {2}}}) is configured with a $\Psi$ ratio of 0.5 to match the complexity of the baseline. Model (\raisebox{.5pt}{\textcircled{\raisebox{-.9pt} {3}}}), with $\Psi=1$ and $K=4$, has $4\times$ the complexity of the baseline model (\raisebox{.5pt}{\textcircled{\raisebox{-.9pt} {1}}}). \textbf{(A)} Training set accuracy. \textbf{(B)} Test set accuracy. \textbf{(C)} $\Psi$ ratio relative to the baseline ResNet-18.   
    \textbf{Right (D-F):} Same layout and analysis, but for sparse models (3 trials; standard deviation shown).} \label{fig:imgnet}
\end{figure}

We begin by employing the Resnet-18 network~\cite{he2016deep} as a baseline point neuron-based model, a widely utilized computer vision model. For this set of experiments we modify the Resnet-18 network architecture as described above to replace typical point neurons with dendritic neurons. Since the dimensionality of our network's input and output remains fixed, the input and output layers are addressed differently, as detailed in ~\nameref{method:arch} section. 

The outcomes of this experiment are shown in Fig.~\ref{fig:imgnet}. We compare models with three levels of complexity. The light-blue dashed curves represent experimental results obtained from various models with a complexity level equal to that of the standard ResNet-18 model. The leftmost data point (\raisebox{.5pt}{\textcircled{\raisebox{-.9pt} {1}}}) corresponds to the standard ResNet-18 model, which serves as a baseline. Subsequent data points to the right denote dendritic models with $K$ values of 4, 16, and 64, respectively. Concomitantly, these models' $D$ values have been adjusted to 1/2, 1/4, and 1/8 of the original model's values, respectively. Given the models we study here are convolutional neural networks, we change $D$ by scale up/down number of channels in network layers. By maintaining this configuration, four models on the same curve have $\Psi$ of $1, 0.5, 0.25, 0.125$ from left to the right, as shown in panel C, all while preserving equivalent parametric and computational complexities (see Appendix~\ref{app:complexity} for a detailed complexity comparison between models).

The solid magenta curves represent data from models where $D$ is doubled compared to the experiments from the light-blue curve. Similarly, the brown dashed curves illustrate models where $D$ are quadrupled. For the brown dashed curve, the point neuron based model at the left end of the curve has $4$ times the number of neurons ($\Psi=4$ in panel C) for each layer as compared to the standard Resnet-18 model (\raisebox{.5pt}{\textcircled{\raisebox{-.9pt} {1}}}). At the right end of the brown curve, we can see the dendritic model, which is equipped with $K=64$ dendrites per neuron, thus having just one eighth of neurons ($\Psi=0.125$) 

Our analysis yields a particularly intriguing result concerning the communication cost of dendritic neuron models compared to point neuron-based models. Specifically, we find that for models of equivalent computing complexity, a dendritic neuron model achieves comparable performance to a point-neuron-based model when $\Psi$ is greater than or equal to $0.25$. 

This reduction in $D$ offered by adding dendrites can significantly reduce the communication cost between neurons in artificial neural networks. 

To obtain a more complete picture, we also compare models with the same number of neurons. The results are illustrated in Fig.~\ref{fig:AK} (Appendix).

\subsubsection{Sparse models on ImageNet}
Thus far we developed dendritic neuron-based models with significantly reduced communication costs, as measured by $D$. These dendritic neuron-based models can also achieve similar (or slightly better) performance compared to corresponding point-neuron-based models of the same computing complexity, in terms of both model expressivity and generalization performance. 

This may seem contradictory to earlier works~\cite{poirazi_impact_2001,wu2018improved}, where dendritic neuron-based models showed higher capacity/expressivity than point-neuron-based models of the same parametric complexity. One might argue that the models we evaluated thus far are non-sparse, which is not biological and differs from the models evaluated in earlier studies. As such, it is essential to also investigate the influence of sparsity on model behavior.

As illustrated in right side panel of Fig.~\ref{fig:imgnet}, we study sparse models with $85\%$ of parameters pruned. Although sparsity generally reduces performance, we observe the same trend as in the non-sparse case: models with different dendritic numbers $K$ but equal computing complexity show little performance difference between dendritic and point-neuron architectures as long as $\Psi$ is above $0.25$. This reinforces our earlier findings and highlights communication efficiency as the key advantage of dendritic neuron models.

\subsubsection{Additional Empirical Verification}
To further substantiate our findings, we conducted additional experiments using a diverse array of model architectures and datasets. This analysis included an assortment of models encompassing those lacking residual connections, as well as those that leverage transformer-based architectures. For the sake of clarity, we have included these additional results in the Appendix~\ref{app:supp_ML}. Similar conclusions are draw from these supplementary experiments.

\subsection{Local communication cost analysis}\label{comm}

Our analysis demonstrates that incorporating dendritic structures into neural networks significantly reduces the required neuron count ($D$) without sacrificing performance, provided the $D$ is sufficiently large. In biological brains, fewer neurons correspond to reduced volume and connectivity, potentially decreasing both neuronal soma volume and the white matter, which consists predominantly of long-range axons, and constitutes approximately $60\%$ of the brain's total volume~\cite{braitenberg2013cortex}. For ANNs running on hardware such as GPUs, reducing $D$ cuts costs by limiting data transfers to off-chip memory and decreasing memory usage for hidden layer activations during inference. We define this neuron-count-related cost $D$ as the \textit{inter-layer communication cost}.

Although dendritic architectures lower $D$, Fig.~\ref{fig:network} shows that dendritic neurons require more synaptic connections per input neuron to match point-neuron model complexity/performance. Thus, communication costs must also account for connecting these additional synapses (weights).

Fig.~\ref{fig:cost_illu} provides a breakdown of communication costs in both biological networks and ANNs, dividing costs into three parts: intra-neuron aggregation cost ($C_A$), inter-layer communication cost ($D$), and intra-layer signal propagation cost ($C_E$). Costs $C_A$ and $C_E$ are measured by the signal's travel distance, while $D$ reflects neuron count.

The top panel of the figure depicts the communication costs in biological neural networks. Here, \(C_A\) corresponds to the cost of aggregating outputs from the dendritic tree and sending them to the cell body. Due to the challenge of separating \(C_A\) from computation costs, this is not intended to be accurately defined. \(D\) represents the long-range communication cost, while \(C_E\) denotes the cost of axonal activation reaching the target synapses.

The bottom panel illustrates the same metrics for ANNs, assuming inference is performed on a mesh of processing elements (PEs). The left section shows aggregation costs ($C_A$), the middle represents inter-layer costs ($D$, including memory storage considerations), and the right indicates intra-layer costs ($C_E$), measured by path length between PEs.

\begin{figure}[!ht]
    \centering
    \includegraphics[width=0.99\linewidth]{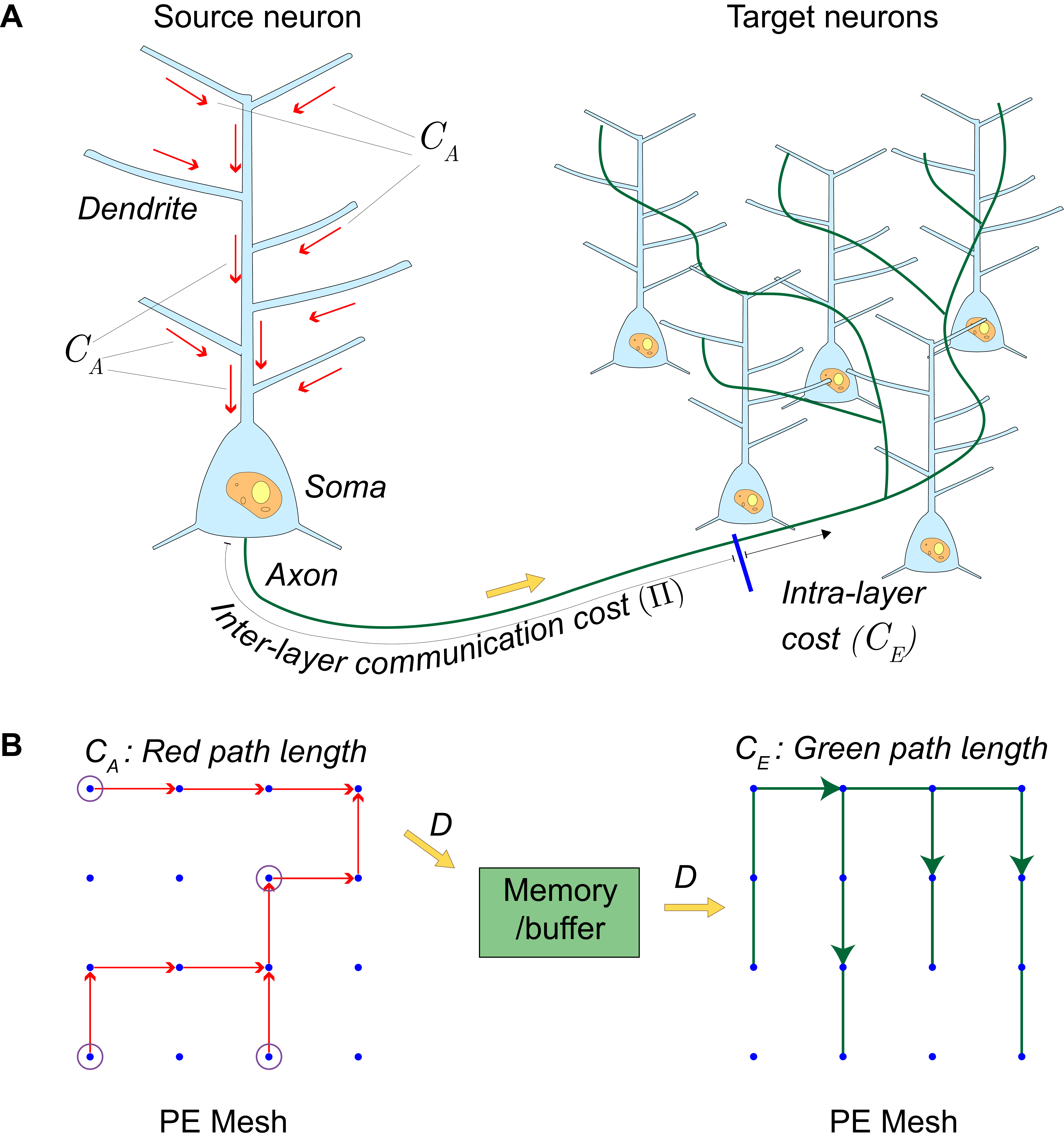}
    \caption[]{Illustration of the three communication cost metrics ($C_A$, $D$, and $C_E$) for (A) Biological Neural Networks and (B) ANNs. In this context, \(C_A\) represents the communication cost associated with aggregating synaptic inputs; \(D\) denotes the inter-layer communication cost, and \(C_E\) signifies the expense related to signal propagation to each synapse (weight). We measure \(C_A\) and \(C_E\) by the total path length over which the signal traverses. For the ANN models, we assume model inference is performed on a mesh of processing elements (PEs). Each blue dot represent one PE unit in the mesh. It's important to note that the division into these three metrics is not intended to be exact.}
    \label{fig:cost_illu}
\end{figure}

Returning to Fig.~\ref{fig:network}, we emphasize the importance of considering communication costs beyond inter-layer interactions. The point-neuron model receives and outputs data of dimension $D$. In contrast, the dendritic neuron model preserves complexity by using fewer neurons ($\hat{D}=D/\sqrt{K}$) but distributing inputs across more dendrites ($D\sqrt{K}$). In this case $D$ and $\hat{D}$ corresponds to the number of neurons for the network layer. Although both model architectures process the same total number of inputs ($D^2$), differences in neuron and dendrite configurations significantly influence communication costs—captured by $C_A$, $C_E$, and $D$—as detailed in the following analysis.

\subsubsection{Cost estimation for a biological neuronal network}

We quantified the wiring (communication) costs, \( C_E \), required to connect $\hat{D}=D/\sqrt{K}$ input neurons to \( D \cdot \sqrt{K} \) synapses each, resulting in a total of \( D^2 \) synapses in biological neuronal networks. These costs were evaluated across various dendritic counts per neuron, \( K = \{1,4,16,64\} \), assuming synapses are spatially distributed either in a two-dimensional (2D) plane or a three-dimensional (3D) volume. Both empirical measurements and fitting with theoretical predictions are presented in Fig.~\ref{fig:cost_comm}A (2D case) and Fig.~\ref{fig:cost_comm}B (3D case). The special case \( K = 1 \) corresponds to the point-neuron model. Clearly, the results demonstrate a significant advantage in adopting dendritic neurons. Further methodological details are provided in Section~\ref{method:comm:bio}.

We do not attempt to estimate the impact of having dendritic neurons on $C_A$ for biological neuronal networks due to the scarcity of biological data and difficulty in separating the computing cost from the aggregating communication cost. However, we speculate that the cost of signal aggregation in biological neurons will be mostly dependent on the number of inputs, and thus adopting a nonlinear or linear dendrite would not significantly affect this type of cost.

\begin{figure}[!ht]
    \centering
    \includegraphics[width=0.99\linewidth]{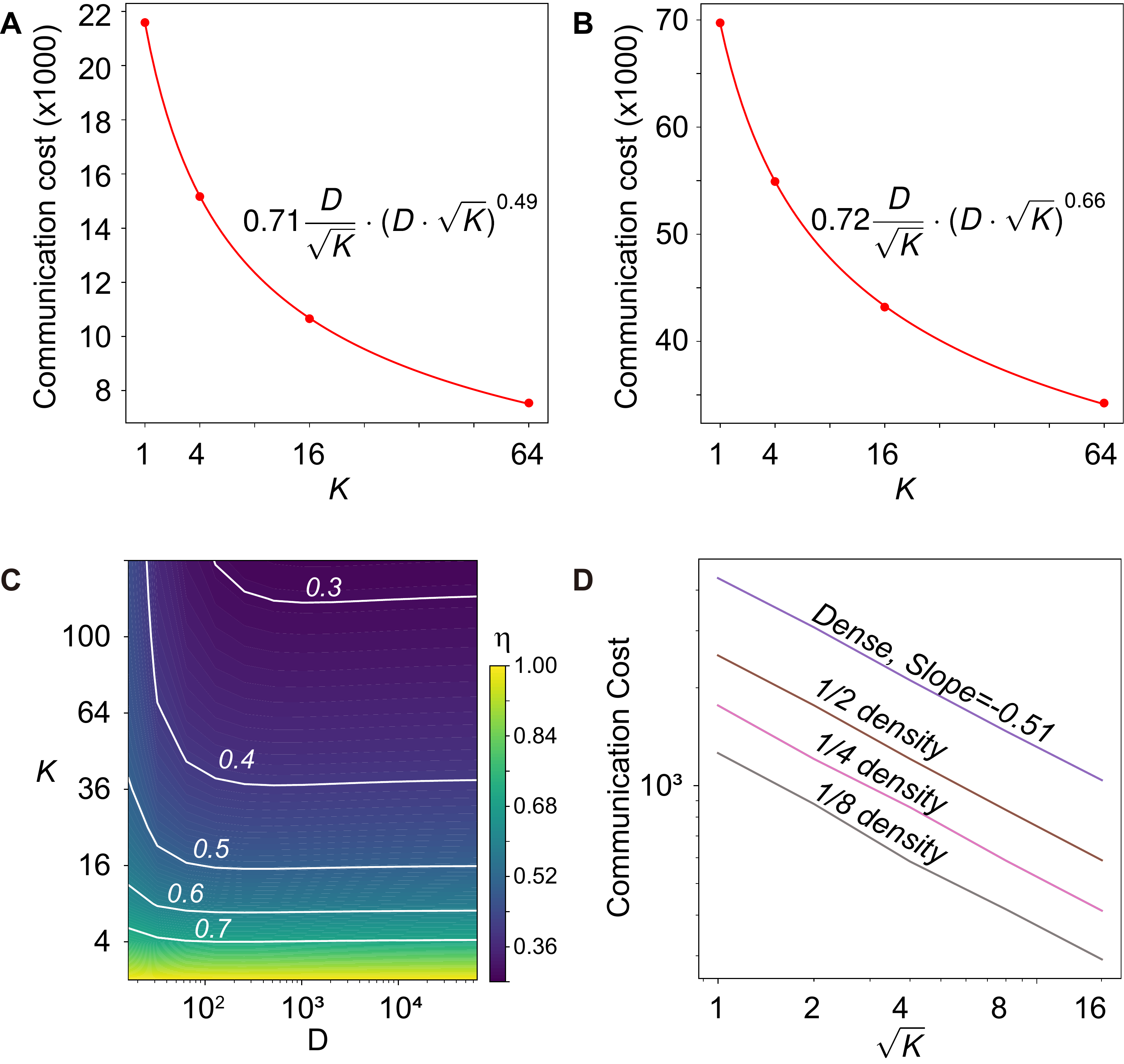}
    \caption[]{(A, B) Estimation of signal propagation costs $C_E$ for a biological network layer with a varying number of dendrites per neuron ($K$) and a baseline network of dimension $D=1024$. Post-synaptic targets sampled from (A) a unit square. (B) from a unit cube. In each panel, the curve and its corresponding equation are fitted to the data points.
    (C,D) Estimation of signal propagation costs for a ANN layer. (C) Topographic representation of the ratio $\eta=({\hat{C}_A+\hat{C}_E})/({C_A+C_E})$: The visualization highlights the influence of the variations in $D$ and $K$ on the $\eta$. (D) demonstrates the variations in \(\hat{C}_E\) as a function of \(\sqrt{K}\), and levels of connection sparsity. The axes are depicted on a logarithmic scale. When \(K=1\), the models are based on point neurons. For this experiment, a \(D\) value of 256 was utilized. The slope is obtained from fitting a line to the logarithm of \(C_E\) against the logarithm of \(\sqrt{K}\).
}
    \label{fig:cost_comm}
\end{figure}

\subsubsection{Cost estimation for an artificial neural network}

We analyzed the communication cost of artificial neural networks (ANNs) using a simplified parallel architecture model (See Appendix~\ref{app:comm_complexity} for details of analysis). Our results show that incorporating dendritic neuron structures into ANN models can significantly reduce on-chip communication cost compared to traditional point neuron models. Specifically, for fixed computational complexity, dendritic models consistently exhibit lower communication costs as the number of dendrites per neuron \(K\) increases. 

Figure~\ref{fig:cost_comm}C shows that the ratio of communication costs \(\eta\) between dendritic and point neuron-based models decreases with increasing \(K\). Moreover, as shown in Appendix~\ref{app:comm_complexity}, in most configurations, the communication cost for dendritic models is dominated by \(\hat{C}_E\), especially for large input dimensions \(D\), which is common in practice.

Further analysis reveals that \(\hat{C}_E\) decreases with increasing model sparsity and increasing \(K\), following a negative power-law relationship with \(\sqrt{K}\), specifically \(\hat{C}_E \propto K^{-0.51}\) as shown in Figure~\ref{fig:cost_comm}D. This closely matches the simplified theoretical form of \(\hat{C}_E\) derived in the Appendix (Eq.~\ref{Eq: C_E_dendrite}), which includes a \(K^{1/4}\) scaling factor. These findings suggest that dendritic neurons can effectively reduce communication costs in sparse ANNs.

\subsection{Reducing Memory Access Cost During Training and Inference on GPU}
\subsubsection{Model inference}
We also extended our analysis to modern GPU-based architectures. Theoretical estimates indicate that dendritic neurons can reduce global memory access and improve efficiency by a factor of \(\sqrt{K}\), highlighting their potential for improving performance in realistic parallel hardware settings. We also verify the theoretical result with empirical experiments. Further details are provided in Appendix~\ref{app:gpu}.

\subsubsection{Model training}
We have demonstrated that dendritic architectures can reduce communication costs during model inference. Can they similarly decrease memory communication costs during training?

At first glance, dendritic models require storing more intermediate activations during training compared to standard models. For example, in Fig.~\ref{fig:network}, a dendritic neuron-based layer stores $16+4=20$
intermediate activations, while a point neuron-based layer only stores 8 activations. Counting each activation value once (as obtaining post-activation values from pre- incurs minimal cost), this translates to 320 bits (dendritic neuron) versus 128 bits (point neuron) per layer using 16-bit floats.

However, memory costs can be significantly reduced by leveraging gradient properties of commonly used activation functions (ReLU, Leaky ReLU, GELU, etc.). Naively, backpropagation through dendritic neurons requires storing intermediate activations $h$ and dendritic pre-activations $\hat{d}_j$ (see Equations~\ref{equ:den} and~\ref{equ:dendrite}). Examining the gradient computations for a dendritic network layer with input $\mathbf{x}$, dendritic pre-activation $\hat{d}$, and neuron-layer output $h$:

\begin{align} \frac{\partial h}{\partial \mathbf{w}_j} = \sigma'(\hat{d}_j)\cdot\mathbf{x},
 \
\frac{\partial h}{\partial \mathbf{x}} = \sum_{j=1}^{K}\sigma'(\hat{d}_j)\cdot\mathbf{w}_j
 \end{align}

we observe gradients calculation here does not depend on $\hat{d}$ or $d$ but $\sigma'(\hat{d}_j)$. 
For ReLU, this derivative is binary (0 or 1), allowing representation with a single bit per dendrite. Since only $h$ (not $\hat{d}_j$) is needed for gradient computations of subsequent layer (in forward direction). It is suffice to store just one bit per dendrite. In the example of Fig.~\ref{fig:network}, this reduces memory from 320 bits (16-bit floats) to only $16 + 4 \times 16 = 80$ bits, even less than the point neuron-based model (128 bits). Similar reductions are achievable for related piecewise activation functions (e.g., Leaky ReLU, Parametric ReLU, and ReLU6).

Non-piecewise activation functions (e.g., GELU, ELU, SELU) may require slightly higher precision. However, their gradient values typically span limited ranges, possibly enabling efficient storage using only a few bits per dendrite (e.g., two bits). Given the robustness of neural network training to gradient noise~\cite{chakrabarti2019backprop}, this approximation is likely acceptable in practice. A detailed investigation is left for future work.

\section{Discussion}\label{discuss}

This work was inspired by biological neurons, which aggregate signals from dendrites into a single output at the soma. Dendrites integrate inputs nonlinearly through voltage-gated channels and receptors, we designed neural network units that mimic these characteristics. 

Treating dendrites as individual point-neurons (as previously proposed~\cite{poirazi_pyramidal_2003}), the idea of aggregating neuronal outputs is common in deep learning, such as spatial pooling in convolutional networks~\cite{lecun1998gradient} and Maxout networks~\cite{goodfellow_maxout_2013}, to enhance performance and reduce computational complexity. Prior studies employing dendrite-inspired pooling strategies reported superior computational and discriminatory capabilities compared to linear integration methods~\cite{wu2018improved, poirazi_impact_2001,chavlis2025dendrites}.

Several related studies include Naud’s sparse neuron ensembles~\cite{naud2018sparse}, Sezener's Dendritic Gated Network (DGN)\cite{sezener2021rapid}, and Iyer et al.’s incorporation of active dendrites into ANNs\cite{iyer_2022}. Naud showed neuron ensembles efficiently communicate combined signals from multiple sources, albeit through a different mechanism. Sezener emphasized performance without exploring how dendrite count affects efficacy. Unlike Sezener, our study evaluates dendritic efficiency, communication costs, and complexity. Iyer et al. focused on shallow ANNs and dendrites' role in continual learning, contrasting with our emphasis on dendritic efficiency in deeper networks.

We demonstrate dendritic neurons substantially improve communication efficiency as network size scales. Typically, increasing network width scales neuron counts ($D$) with the square root of parameters, while depth requires linear scaling. Our results demonstrate dendritic architectures can significantly increase parametric complexity without increasing $D$. This substantially reduces inter-layer communication, lowering data-transfer costs within computing chips, and may have analogous biological benefits~\cite{levy2021communication}, although further exploration of biological wiring costs remains necessary.

Our analysis of dendritic architecture’s communication advantages considered only wire length, excluding wiring volume. Earlier research suggests dendrites confer significant volume savings~\cite{chklovskii2004synaptic}. Accounting for wiring volume would likely enhance our architecture’s benefits but, due to limited biological data, this aspect is left for future research.

Also relevant to our study is the concept of small-world networks, which achieve communication efficiency through specific connectivity patterns~\cite{watts1998collective,latora2001efficient}. We instead focus on local dendritic nonlinearities to minimize communication costs.

Our findings carry theoretical and practical implications. Theoretically, dendritic architectures suggest widening networks by enhancing feature complexity rather than solely increasing inter-layer communication. Practically, dendritic models outperform point-neuron models at equal communication budgets, substantially reducing memory access, especially beneficial for large-batch inference. Our results predict dendritic designs can reduce on-chip communication costs, potentially informing neural accelerator design.

Our analysis on reduced memory costs from dendritic neural network training is limited to ANNs. Whether this benefit translates to biological dendritic neurons remains open, given limited understanding of biological learning mechanisms, and is deferred to future research.


We also lack complete understanding of why dendritic channel-sharing matches or exceeds conventional model performance. One plausible explanation involves interpreting dendritic pooling as low-rank approximations of large weight matrices. Further investigation is necessary.


Notably, our dendritic models apply a single nonlinear layer solely at dendrites. Preliminary results indicated minor performance gains when adding additional somatic nonlinearities, particularly with large number of dendrites. However, this is not included here, as the primary study focus is efficiency rather than incremental performance improvements. Exploring diverse nonlinearities and advanced architectures in dendritic neurons remains intriguing future work.

This study's analysis of active dendrites relied on machine learning experiments employing a rate-based model. It is crucial to acknowledge that this methodological choice may introduce limitations, particularly when comparing the findings to those derived from spike-based models.

Finally, our findings parallel evolutionary patterns in biological brains, where complex dendritic structures emerge in larger neural systems due to increased computational demands~\cite{stuart2016dendrites}. Integrating dendritic neurons into artificial networks may thus reflect fundamental biological principles, offering insights for efficient and scalable neural network design.

\section{Methods}\label{method}
\subsection{Datasets for machine learning experiments}

The present study leverages three commonly used datasets: ImageNet, CIFAR-100, and LibriSpeech, for model training and evaluation. These datasets are commonly served as benchmarks in deep learning research.

\textbf{ImageNet Dataset:} For this study, we use the ILSVRC 2012 subset of the ImageNet dataset, which consists of 1.2 million training images and 50,000 validation images from 1,000 categories~\cite{deng2009imagenet}. The images vary in size and are resized to a fixed resolution of 224x224 pixels for uniformity, per the standard ResNet procedure~\cite{he2016deep}. The typical data augmentation techniques, such as random cropping, random horizontal flipping, and color jittering, were applied during training to enhance the model's generalization ability.

\textbf{CIFAR-100 Dataset:} The dataset consists of 60,000 32x32 color images in 100 classes, with 600 images in each class. There are 50,000 training images and 10,000 test images~\cite{krizhevsky2009learning}. Like the ImageNet data processing, we followed the typical data augmentation procedure~\cite{he2016deep}. 

\textbf{LibriSpeech dataset:} The dataset is a publicly available English speech corpus for Automatic Speech Recognition (ASR) training and evaluation from the LibriVox project's audiobooks. It consists of 1000 hours of transcribed speech, divided into training, development, and testing subsets~\cite{panayotov2015librispeech}. The experiment utilizing this dataset can be found in the Appendix E.

\subsection{Model architectures}\label{method:arch}

In this study, we primarily used the ResNet-18 architecture as the baseline model. ResNet-18 is an 18-layer deep residual neural network, a seminal model proposed by He et al.~\cite{he2016deep}. The baseline configuration of ResNet-18 encapsulates an initial convolutional layer, followed by four residual blocks, each of which consists of two convolutional layers. This pattern constitutes the primary structure of our working model; in contrast to the original ResNet-18 model, our adapted architecture positions the shortcut connection after the ReLU (Rectified Linear Unit) activation function. This modification is imperative to ensure the compatibility of the dendritic structure with the model architecture.

For experiments on scaling up networks, we scaled up each network layer by the same designated factor except for the input and output of the model. For models with dendritic neurons, we replaced neurons in the standard model with dendritic neurons with $K$ dendrites as specified by the experiment setting, except for the input and output layers of the model. To maintain the uniform model complexity scaling throughout the model, we equip the input layer and the penultimate layer of the model with neurons of $\sqrt{K}$ instead of $K$ dendrites. The same setting is also employed in experiments designed to compare models that share identical inter-layer communication costs.

For models trained on CIFAR-100, we observed training instability. Therefore we clipped the gradient norm to $1.0$ during model training. We also added an extra batch norm to each dendrite to improve model stability. This additional batch norm can be fused with the previous layer and thus will not add extra computation burden at the inference stage. 

In addition to models based on the ResNet-18 architecture, we have corroborated our findings using a model devoid of shortcut connections. This strategy ensures that the benefits observed are not strictly confined to a particular architecture. The configuration of this model is delineated in Appendix E, where the corresponding experimental outcomes can also be found.

Moreover, our experimentation extended to the transformer-based model. Within this model, the standard feedforward layers are substituted with network layers based on dendritic neurons. Comprehensive details pertaining to this modification can be found in Appendix E.

\subsection{Model training}

We trained all models with a cosine learning rate decay schedule and the SGD optimizer with a momentum of 0.9.

For ImageNet with dense ResNet models, the learning rate was initialized at $0.4$ (instead of $0.1$ to compensate for the batch size used for training), and models were trained for 120 epochs, including two warm-up epochs with a learning rate of $0.04$. Weight decay was set to $1\times 10^{-4}$. A batch size of 1024 was employed, and the training was distributed across 8 GPUs. 

For ImageNet with sparse ResNet models, the models were trained for 200 epochs with an initial learning rate of 0.1 and 2 warm-up epochs at a learning rate of 0.01. The weight decay parameter was set to $1\times 10^{-4}$. To achieve a sparse ratio of 85\%, we applied L1-unstructured global pruning in 5 rounds, conducted between epochs 40 and 140. Subsequently, the models were trained for an additional 60 epochs.

Finally, for CIFAR-100 models, we trained them for 200 epochs with a learning rate of 0.05, including two warm-up epochs at a learning rate of 0.005. A batch size of 64 was utilized, and the weight decay parameter was set to $5\times 10^{-4}$.

Our investigation emphasizes the comparative analysis of the performance of various models under identical training conditions, facilitating an equitable assessment of the distinct capabilities of each model. Consequently, all models within the comparison group undergo training with the same hyper-parameters, barring the requisite architecture adjustments. Further details concerning the experiments can be found in the accompanying source code. 

\subsection{Communication cost analysis}\label{method:comm}
\subsubsection{Biological neural network}\label{method:comm:bio}
We modeled a baseline network layer with \(D\) input and \(D\) output neurons, resulting in \(D^2\) synapses. For each input neuron, \(D\sqrt{K}\) synaptic targets were randomly distributed within either a unit square (2D) or a unit cube (3D), where \(K\) represents the number of dendrites per neuron and is varied across \(K=1, 4, 16, 64\).

To estimate the wiring length, we computed the Euclidean minimal spanning tree over each synapse set, using the method described by Steele et al.~\cite{steele1989worst}. This was repeated 10 times to obtain an average total path length. The communication cost \(C_E\) was then calculated as the product of \(D/\sqrt{K}\) and the mean path length. Finally, we fitted the results to a function of the form \(\alpha \cdot D/\sqrt{K} \cdot (D\sqrt{K})^\beta\), allowing us to extract the scaling exponent \(\beta\) and compare it to theoretical expectations.

\subsubsection{Artificial neural network}\label{method:comm:ANN}
We adopted a simplified parallel explicit communication model (PECM), inspired by Dally et al.~(2022), to estimate data movement costs in ANN inference hardware. The model assumes that computation occurs on a 2D grid of processing engines (PEs), interconnected via an on-chip network (NoC) within a unit square. This abstraction captures essential features of neuromorphic and parallel architectures used in real-world ANN hardware~\cite{akopyan2015truenorth,davies2018loihi,ma2024darwin3}.

Communication costs were analyzed for both point neuron and dendritic neuron-based models, with derivations provided in the Appendix~\ref{app:comm_complexity}. The focus was on two key components: aggregation cost \(C_A\) and external communication cost \(C_E\), and their dendritic counterparts \(\hat{C}_A\) and \(\hat{C}_E\).

We evaluated the communication cost ratio \(\eta\) between dendritic and point neuron models across a range of parameters: different values of input dimensionality \(D\) for point neurons and varying numbers of dendrites per neuron \(K\) for dendritic neurons. The impact of sparsity was also assessed by analyzing \(\hat{C}_E\) under varying sparsity levels and dendrite counts.

\subsection{Code availability}
The entirety of the code used to produce the findings presented herein will be openly accessible to the public upon publication. 
\bibliography{sn-bibliography}

\begin{appendices}


\section{Proof of Theorem~\ref{thm:1}}
As illustrated in Eq.~\ref{equ:entropy}, the left-hand term—representing the information from the pooled neuron output—is bounded by the information present in the dendrites being pooled. Following shows a comprehensive proof of this theorem.

\begin{theorem}
\label{thm:1}
The entropy of the sum (neuron output) of random variables (dendritic outputs) $\rd_1, \rd_2, \dots, \rd_K$ is less than or equal to the joint entropy of these random variables. The relation between the two is given as:

\begin{equation}
\label{equ:entropy}
    H\left(\sum_{j=1}^{K}\rd_j\right) = H(\rd_1, \dots, \rd_K) -  H\left(\rd_1, \dots, \rd_K \mid  \sum_{j=1}^{K}\rd_j\right)\,.
\end{equation}
\end{theorem}
In order to prove Theorem~\ref{thm:1}, we first prove the following lemma.
\begin{lemma}
The conditional entropy of $\sum_{j=1}^{K}\rd_j$ given $\rd_1, \rd_2, \dots, \rd_K$ is zero, i.e.,
\label{lem:1}
\begin{equation}
    H\left(\sum_{j=1}^{K}\rd_j  \mid \rd_1, \rd_2, \dots, \rd_K \right) = 0\,.
\end{equation}
\end{lemma}
\begin{proof}
    If the values of $\rd_1, \rd_2, \dots, \rd_K$ are known, then the value of $\sum_{j=1}^{K}\rd_j$ is also known. Therefore, the statement is intuitively true. For discrete random variables $\rd_i$, a formal proof can be presented as follows.
    \begin{align}
    &{\phantom{=.}} H\left(\sum_{j=1}^{K}\rd_j  \mid \rd_1, \rd_2, \dots, \rd_K \right) \nonumber\\
    &=\sum_{d_1,\dots,d_K}p(d_1,\dots,d_K)H\left(\sum_{j=1}^{K}\rd_j  \mid \rd_1=d_1, \dots, \rd_K=d_K \right) \nonumber \\
    &= \sum_{d}0 = 0\,.
    \end{align}
\end{proof}
We can now prove Theorem~\ref{thm:1}, by first making use of a relationship between joint entropy and conditional entropy~\cite{coverelements}.
\begin{proof}[Proof of Theorem~\ref{thm:1}:]
The joint entropy of $\rd_1, \rd_2, \dots, \rd_K$ and $\sum_{j=1}^{K}\rd_j$ is:
\begin{align}
\label{eq:1}
&{\phantom{=.}}H\left(\rd_1, \dots, \rd_K  , \sum_{j=1}^{K}\rd_j \right)  \nonumber\\
&= H\left(\sum_{j=1}^{K}\rd_j\right) + H\left(\rd_1, \dots, \rd_K  \mid \sum_{j=1}^{K}\rd_j \right)  \nonumber\\
&= H(\rd_1, \dots, \rd_K) + H\left(\sum_{j=1}^{K}\rd_j  \mid \rd_1, \dots, \rd_K \right) \,. 
\end{align}
Therefore,
\begin{align}
    &{\phantom{=.}}H\left(\sum_{j=1}^{K}\rd_j\right)  \nonumber\\
    &= H(\rd_1, \dots, \rd_K) + H\left(\sum_{j=1}^{K}\rd_j  \mid \rd_1, \dots, \rd_K \right)  
     - H\left(\rd_1, \dots, \rd_K  \mid \sum_{j=1}^{K}\rd_j \right) \nonumber\\
    &= H(\rd_1, \dots, \rd_K) - H\left(\rd_1, \dots, \rd_K  \mid \sum_{j=1}^{K}\rd_j \right) \;\;\;\;\;\;\quad\quad\quad\quad (\textit{From Lemma~\ref{lem:1}.})
\end{align}
\end{proof}

Since the conditional entropy $H\left(\rd_1, \dots, \rd_K  \mid \sum_{j=1}^{K}\rd_j \right)$ is non-negative, the upper bound of $H\left(\sum_{j=1}^{K}\rd_j\right)$ is $H(\rd_1, \rd_2, \dots, \rd_K)$.

\section{Computing and parametric complexity of models}
\label{app:complexity}
\begin{table}[htb!]
    \centering
    \caption{Complexity data on typical models}
    \label{tab:complexity}
    \begin{tabular}{cccc}
        \hline
        Dendrites/neuron ($K$) & $\Psi=1/\sqrt{K}$ & Computing complexity (MMACs) & \# of Parameters \\
        \hline
        1 (Resnet-18) & 1 & 1,821.63 & 11,689,512 \\
        4 & 1/2 & 1,804.34 & 11,556,200 \\
        16 & 1/4 & 1,799.65 & 11,521,800 \\
        64 & 1/16 & 1,799.37 & 11,512,664 \\
        \hline
    \end{tabular}
\end{table}

Table~\ref{tab:complexity} shows the computing and parametric complexity comparison of models from the light-blue dashed curve of Fig.~\ref{fig:imgnet}. 
We use customized THOP package to calculate the model complexity data where we count two sum operations as one MAC operation.

\section{Derivation of Communication Costs for PE Mesh Architecture}\label{app:comm_complexity}
\subsection*{Point neurons based model}
Our analysis initiates with a model composed of point neurons. As previously mentioned, our investigation focuses on two network layers. We assume that the first layer sends an output of $D$ dimensions to the second layer. For convenience and without loss of generality, we assume that each of the $D$ dimensions originates from one PE on the chip.

In order to arrange $D$ PEs on die area of size $1\times 1$, each PE must have a height and width of $l={1}/{\sqrt{D}}$, resulting in an area size of ${1}/{D}$. Similarly, the second layer is also composed of $D$ PEs of the same size. Consequently, we obtain a grid of $N$ by $N$ PEs with $N=\sqrt{D}$, with a distance of $l$ between the center of each pair of neighboring PEs. See Fig.~\ref{fig:grid}-A for a visual illustration.
\begin{figure}[htb!]
    \centering
\includegraphics[width=0.9\linewidth]{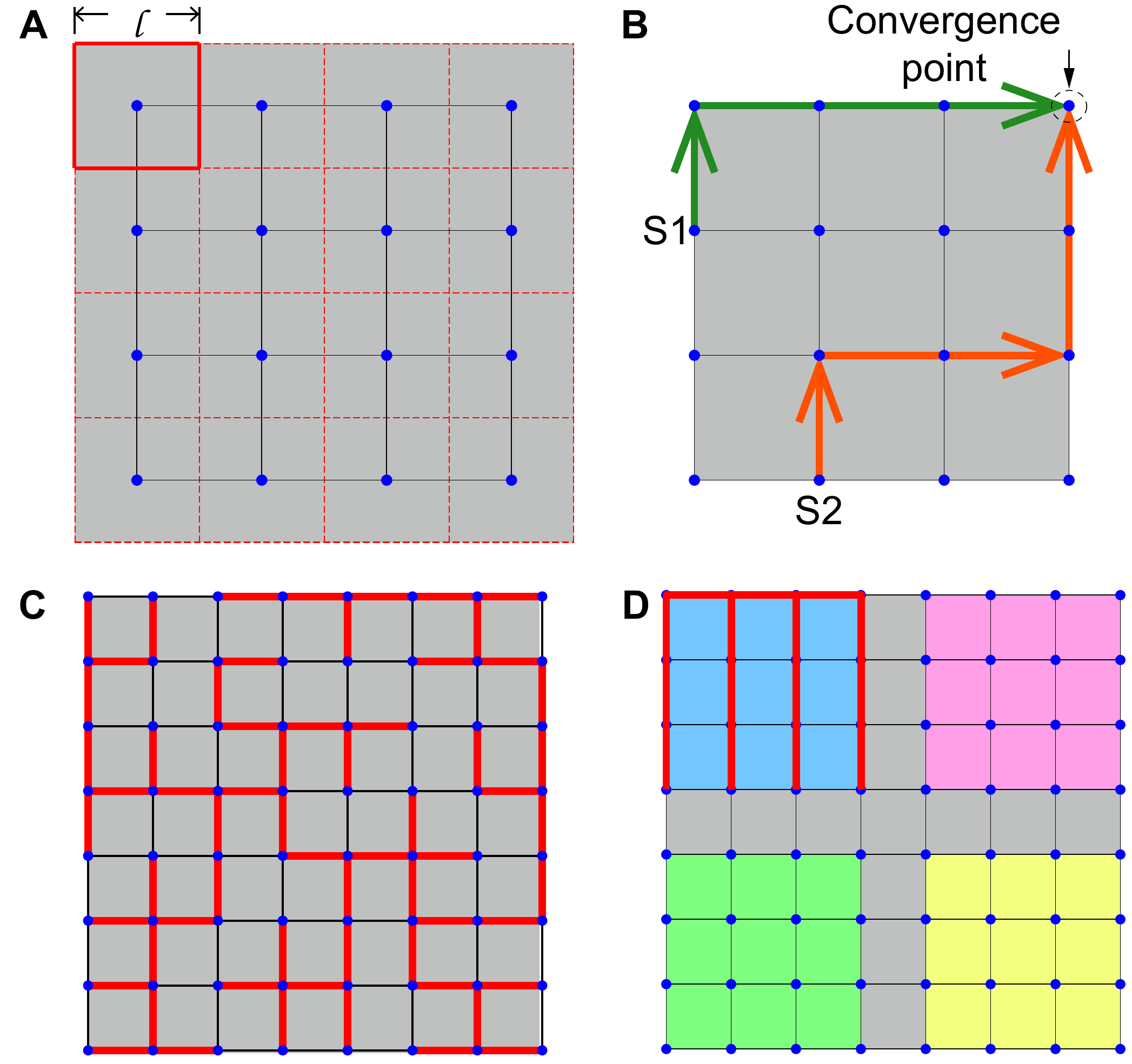}
     \caption{(A) Showcases a 16-unit grid of processing elements (PEs), where each PE has a side length, $l$, computed as $l=1/\sqrt{D}$ or $1/4$ in this example. The boundary of top-left PE is emphasized with solid red lines. For improved clarity, only the grid of central points will be displayed henceforth.(B) Depicts two city-walk paths originating from S1 (green path) and S2 (orange path) leading to a convergence point. The green path has a total length of $4l$, while the red path spans $5l$ in length. (C) Demonstrates a city-walk path with red lines, connecting all points on an $8\times 8$ grid. This route enables data dissemination across the target set with minimal cost. (D) Illustrates four groups of PEs, each color-coded to represent a dendritic neuron with 16 dendrites. Within each group, dendritic outputs are combined to generate a single output. The aggregation path can be assessed using the MRST algorithm, with an example path displayed in the top-left block.}
    \label{fig:grid}
\end{figure}

For this arrangement we have 
\begin{equation}
\label{Eq: CA}
    C_A = D(\sqrt{D}-1)l = D-\sqrt{D} \,,
\end{equation}
as measured with Manhattan distance. Furthermore, an illustrative example of signal propagation within this context is provided in Fig.~\ref{fig:grid}-B. Derivation of Eq.~\ref{Eq: CA} can be found in Appendix~\ref{app:derivation}.

We assess $C_E$ with minimal rectilinear spanning tree (MRST) algorithm~\cite{murty2008graph}. Given a grid of $N \times N$ PEs, the objective is to deliver every dimension of the data to each PE. The MRST algorithm enables us to determine the minimal path length required to connect all PEs, which is $(N^2-1)\cdot l$. An example path is illustrated in  Fig.~\ref{fig:grid}-C. Consequently, we obtain the cost of delivering data as
\begin{equation}
\label{Eq: CE}
    C_E =  (N^2-1)\cdot l \cdot D = (D-1) \sqrt{D} \,.
\end{equation}

\subsection*{Dendritic neuron based model}
As earlier we maintain the number of parameters and floating-point operations (FLOPs) consistent with those in the point neuron model scenario. That is, given that each neuron has $K$ dendrites, one layer of the model under examination will have a total of $M=D\sqrt{K}$ dendrites. As illustrated in Fig.~\ref{fig:grid}-D, every group of $K$ dendrites aggregates to form a single output dimension. Consequently, the first layer will produce an output with a dimensionality of $\hat{D}={D}/{\sqrt{K}}$, which serves to maintain an equivalent computational complexity as the point neuron-based model previously described. We reiterate our assumption that those $\hat{D}$ neurons are arranged in a grid format, specifically of size $\hat{N} \times \hat{N}$, with $\hat{N}=\sqrt{\hat{D}}$.

We postulate that the computation of each dendrite is processed by one PE. In this scenario, the die area is divided into $M$ units, with each unit occupying a specific area. The height and width of this area, denoted by $\hat{l}$, can be calculated as $\hat{l} = 1/\sqrt{M} $. Through this, we arrive at the size of a PE for processing each dendrite being $\frac{1}{D\sqrt{K}}$, which is $1/{\sqrt{K}}$ of the point neuron-based model PE die size. This corresponds to the assumption that a dendrite in this analysis receives a proportion of $1/{\sqrt{K}}$ of the inputs that a point neuron receives. 

In light of the aforementioned derivation, we note that the signal transfer cost, denoted as $\hat{C}_A$, consists of two components. The first component, $\hat{C}_{AG}$, refers to the cost of aggregating dendritic outputs for each neuron. The second component, $\hat{C}_{AA}$, represents the cost of transmitting the aggregated data of all neurons off the die. Their expressions are as follows.
\begin{equation}
\label{Eq: CA_hat}
    \hat{C}_A  = \hat{C}_{AG}+\hat{C}_{AA}
    <\sqrt{D}K^{1/4}+\frac{D}{\sqrt{K}} \,.
\end{equation}
Please see Appendix~\ref{app:derivation_2} for the derivation.

In congruence with the approach adopted for dense models, we also employ the MRST algorithm to estimate the communication cost when dealing with sparse models. Considering the variability in the communication cost due to different sparse connection patterns, we sample a set of 100 random connection patterns for each setting to provide a robust estimate of the average cost.
Akin to the point neuron models, we will not attempt to derive $\hat{C}_I$, although we have the relationship of $D = \sqrt{K} \cdot \hat{C}_I$ under the assumptions of the equivalent parameter/FLOPs count setting.

As for the $\hat{C}_E$ component, note that the second layer receives $\frac{D}{\sqrt{K}}$ inputs and consists of $M$ units. Utilizing the MRST method, the cost associated with one-dimensional input connecting to $M$ units can be computed as $(M-1) \cdot \hat{l}$.
We arrive at 
\begin{equation}
\label{Eq: C_E_dendrite}
    \hat{C}_E= \frac{D}{\sqrt{K}}(D\sqrt{K}-1)\cdot \hat{l} \approx  D^{\frac{3}{2}}/ K^{\frac{1}{4}} \,. 
\end{equation}

\subsection{Derivation of Eq.~\ref{Eq: CA}}
\label{app:derivation}

For simplicity, we place the inter-chip communication junction point at the top-right corner, in the $0$-th row and column. It starts with ID 0, counting from right to left and top to bottom. 
Therefore, the total cost of propagating outputs from every PE to the junction point is:
\begin{align}
\label{Eq: CA_derivation}
    C_A &= \left(\sum_{x,y=0}^{N-1}(x+y) \right)l\nonumber\\
    &= \left(\sum_{x=0}^{N-1}x\sum_{y=0}^{N-1}1+ \sum_{x=0}^{N-1}1\sum_{y=0}^{N-1}y\right)l \nonumber\\
    &= \left(\frac{(N-1)N}{2}N+ N\frac{(N-1)N}{2}\right) l\nonumber\\
    &=N^2(N-1)l \nonumber\\
    &=D(\sqrt{D}-1)\frac{1}{\sqrt{D}} \nonumber\\
    &=D-\sqrt{D} \,.
\end{align}

\subsection{Derivation of Eq.~\ref{Eq: CA_hat}}
\label{app:derivation_2}
\begin{align}
    &\hat{C}_{AG}  =(K-1)\cdot \hat{D} \cdot \hat{l}\nonumber\\
    &\; \; \; \;\; \;\;= \sqrt{D}(K^{1/4}-K^{-3/4}) < \sqrt{D}K^{1/4} \,,\\
    &\hat{C}_{AA}  = \hat{N} \hat{N} (\hat{N}-1) \hat{l}(\sqrt{K})< \frac{D}{\sqrt{K}} \,,\\
    &\hat{C}_A  = \hat{C}_{AG}+\hat{C}_{AA}
    <\sqrt{D}K^{1/4}+\frac{D}{\sqrt{K}} \,.\label{Eq: C_A_dendrite}
\end{align}

\section{Communication cost Analysis for block-wise GEMM computation on GPU}\label{app:gpu}

\subsection{Theoretical analysis}

In this section, we analyze how the adoption of the proposed dendritic structure affects communication costs during neural network inference when compared to a point neuron-based structure on typical GPU-like architectures.

For this part of analysis we follow the notation used by the GPU community as in CUTLASS~\cite{CUTLASS}, which differs from the notation used in the rest of the manuscript.

First, we delineate our setting, assuming a feed-forward network layer. For the standard model, the computation of the layer can be expressed as \[\mC^f = \sigma(\mA \cdot \mB),\]
where \(\sigma\) represents the element-wise nonlinear output function. For clarity, we omit the bias term.

The computational complexity of the nonlinear function $\sigma$ is relatively small compared to that of the matrix multiplication. Therefore, our focus will be on the matrix multiplication \[\mC = \mA \cdot \mB.\] And we have $\mA \in \mathbb{R}^{M\times L}$, $\mB \in \mathbb{R}^{L\times N}$, $\mC \in \mathbb{R}^{M\times N}$.

Similarly, for the dendritic model, we have \[\hat{\mC}^f = \sigma(\hat{\mA} \cdot \hat{\mB})\] before the dendritic aggregation process. The dendritic layer output $\hat{\mC}^o$ is then computed as \[ \hat{C}^o_{i,j}= \sum _{s=1}^K {\hat{C}_{i,(j-1)K+s}}\], where $K$ is the number of dendrites per neuron. 
We have $\hat\mA \in \mathbb{R}^{M\times L/\sqrt{K}}$, $\hat\mB \in \mathbb{R}^{L\times N\sqrt{K}}$, $\hat\mC \in \mathbb{R}^{M\times N\sqrt{K}}$.
As described above, we reduce every $K$ neighboring elements along $N$ dimensions in $\hat{\mC}^f$ into one element, namely the output of the dendritic layer $\hat{\mC}^o$. Given that dendritic aggregation can be performed locally with low cost and the computing complexity of nonlinear functions is small. We again focus on the core matrix multiplication part described as $\hat{\mC}=\hat{\mA} \cdot \hat{\mB}$. Though the aggregation process is important for reducing output communication cost.

To understand the communication cost we need to get some idea about the memory hierarchy of a GPU. The architecture depicted in the Fig. \ref{fig:sm} represents a simplified illustration of a typical GPU processor such as Nvidia A100/H100. In this architecture, the global memory is a central storage resource accessible to all processing elements (PEs) within the processor. The PE units here correspond to Streaming Multiprocessors (SMs) of GPUs.  To reduce the communication cost of accessing the high latency global memory, there is also an on-chip L2 cache that accelerates data I/O of PEs.  Each PE is equipped with its own private shared memory, which is utilized by the tensor cores housed within that PE. Each tensor core within a PE is further equipped with its own private registers for localized computing, ensuring rapid access to frequently used data and further optimizing computational efficiency.

\begin{figure}[htb!]
\centering
\includegraphics[width=0.4\linewidth]{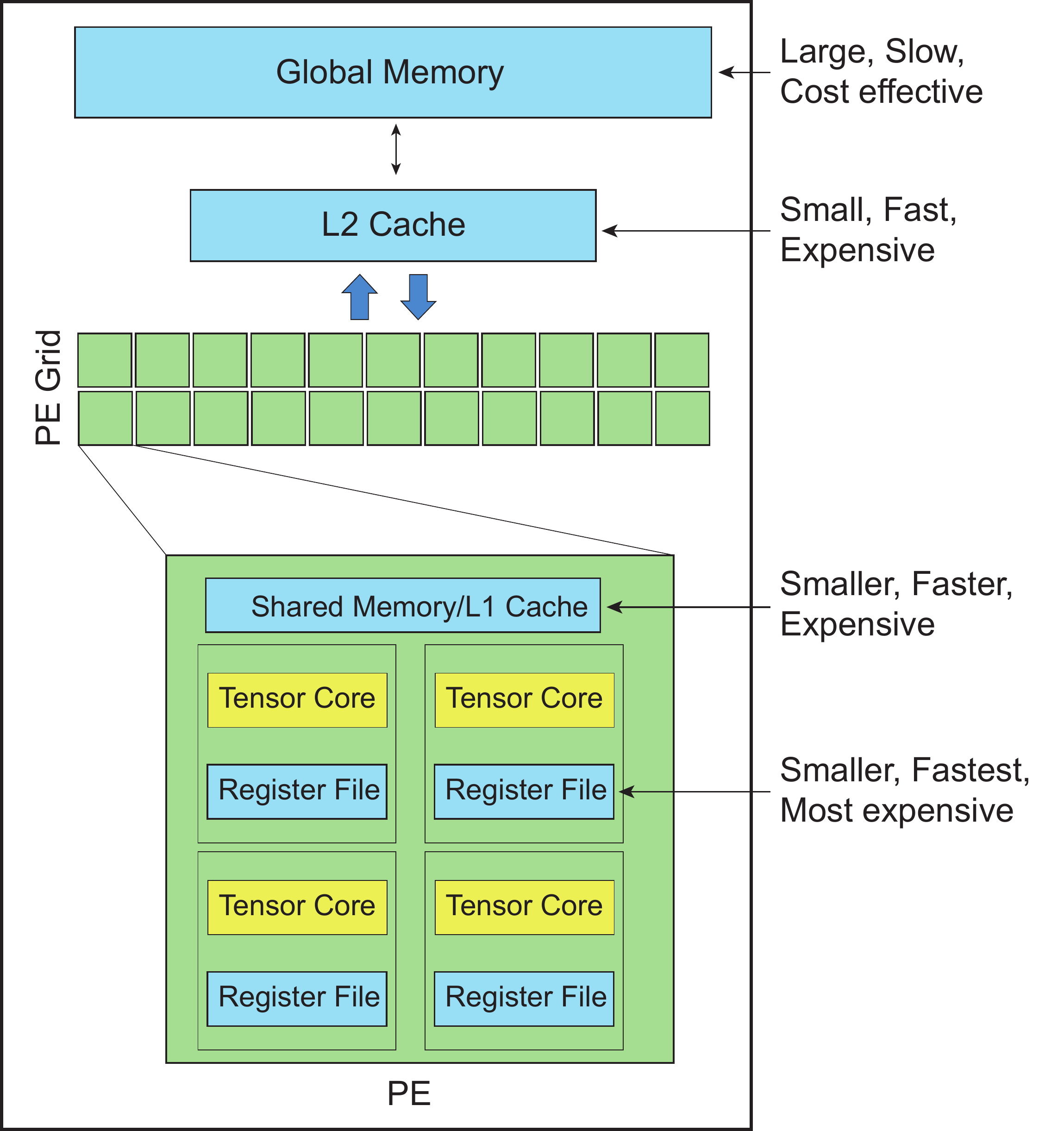}
\caption{Simplified architecture of a GPU processor. The global memory, characterized by being large, slower, and cost-effective, is accessible by all processing elements (PEs) in the processor~\cite{CUTLASS,choquette2021nvidia}. On chip L2 cache enable much faster access to data accessed by PE units. Each PE contains its own private shared memory, which is smaller, faster, and more costly, shared by the tensor cores within that PE. Each tensor core has its own private register file, which is the smallest, fastest, and associate with low communication cost.}
    \label{fig:sm}
\end{figure}

\begin{figure}[htb!]
    \centering
\includegraphics[width=0.5\linewidth]{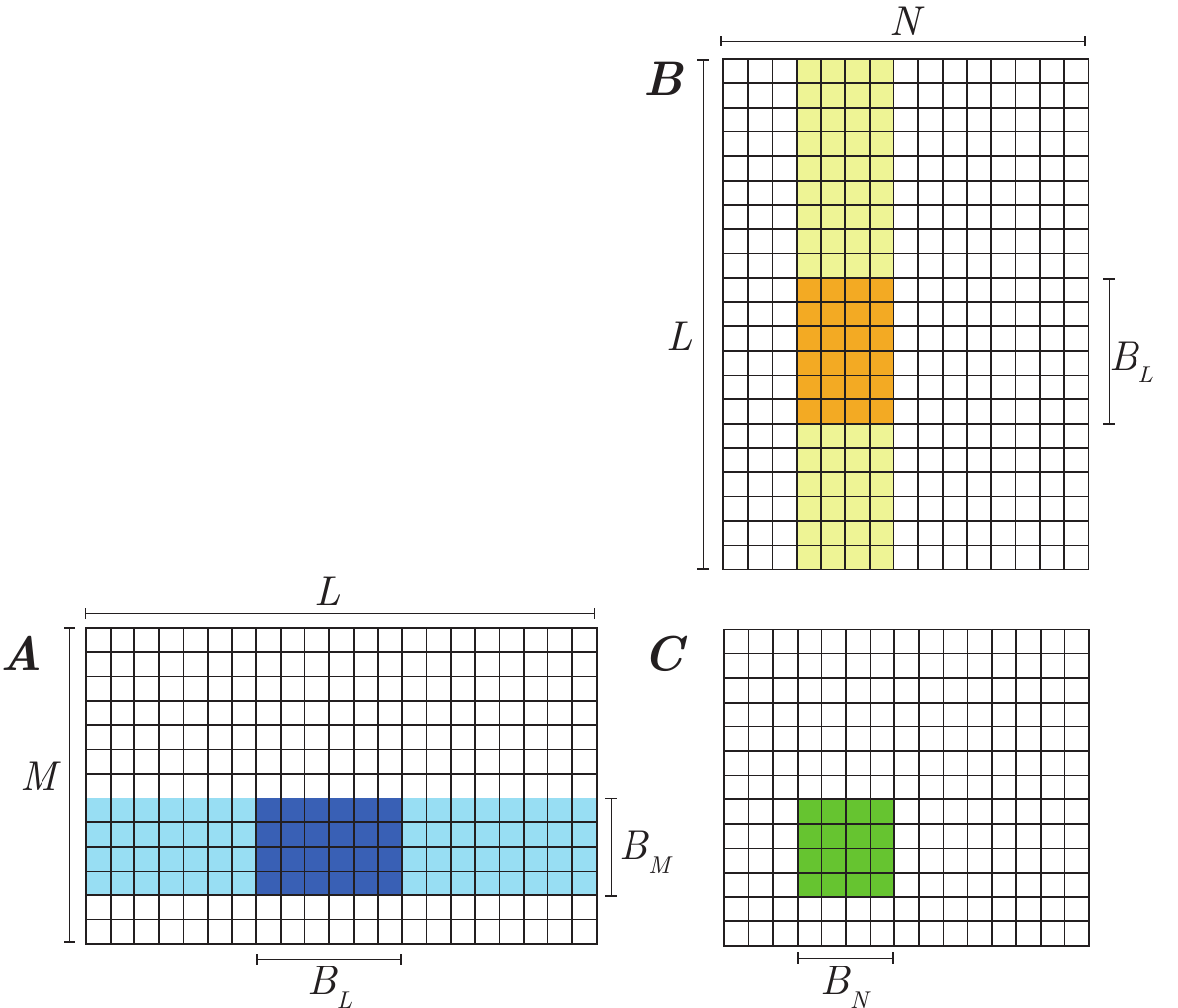}
         \caption{Block-wise General Matrix-Matrix Multiplication (GEMM) illustration. This figure demonstrates the multiplication of matrices $\mC = \mA \cdot \mB$, where $\mA \in \mathbb{R}^{M \times L}$, $\mB \in \mathbb{R}^{L \times N}$, and $\mC \in \mathbb{R}^{M \times N}$. The computation is divided into blocks to optimize performance and efficiency.}
    \label{fig:gemm}
\end{figure}

To avoid diving too deep into GEMM optimization, we refer readers to the CUTLASS documentation for introduction on block-wise general matrix multiply (GEMM)~\cite{CUTLASS}.

The pseudo code for the baseline GEMM is shown in Alg. ~\ref{alg:blockwise-gemm}. We first calculate communication complexity of the baseline models without considering L2 cache.

From Fig.~\ref{fig:gemm}, for each block in $\mC$ of size $B_M$ by $B_N$, we need to read $L(B_M+B_N)$ units of data from global memory. Therefore for computation of the whole $\mC$ we need to read \[(B_M+B_N)\cdot L\cdot \frac{M}{B_M}\cdot \frac{N}{B_N}\] units of data. To write the result matrix $\mC$ to the global memory, the write cost is 
\[M\cdot N.\]

For a dendritic model with $K$ dendrites per neuron, we have $\hat\mA \in \mathbb{R}^{M\times L/\sqrt{K}}$, $\hat\mB \in \mathbb{R}^{L\times N\sqrt{K}}$, $\hat\mC \in \mathbb{R}^{M\times N\sqrt{K}}$. With this new matrix dimensionality we have the following memory read cost for the dendritic model: \[(B_M+B_N)\cdot \frac{L}{{\sqrt{K}}}\cdot \frac{M}{B_M}\cdot \frac{N\cdot {\sqrt{K}}}{B_N} = (B_M+B_N)\cdot L\cdot \frac{M}{B_M}\cdot \frac{N}{B_N}.\] This is the same as the point neuron based model. 
As for the write cost, because we reduce elements in $\hat\mC$ by group of $K$, we have a writing cost of : \[\frac{M\cdot N}{{\sqrt{K}}}.\]

From the above analysis, it is evident that with the same $B_M$, $B_N$ combinations for equivalent point neuron and dendritic neuron based network layers, there is no difference in the total memory read cost. The memory write cost can be reduced by $\sqrt{K}$. This analysis may suggest that adopting dendritic structure can only lead to minor reduction in communication cost given memory read is much larger than memory write. 

The key is to coordinate block processing properly to take advantage of the L2 cache. For this part we refer the reader to the "L2 cache optimization" section of the Triton GEMM tutorial~\cite{tritonlangMatrixMultiplication} and CUTLASS documentation~\cite{CUTLASS} for background information. 

To improve computational efficiency, it is beneficial to take advantage of the sharing of data among neighboring blocks of the matrix $\mC$. This is achieved by computing several adjacent rows of $\mC$, which correspond to the same rows in the matrix $\mA$, as a group. This grouped computation strategy allows for the reuse of input blocks from matrices $\mA$ and $\mB$, minimizing data reloading and maximizing cache utilization. After completing the computation for one group, the process then transitions to another group. This approach not only streamlines data access patterns but also significantly reduces memory overhead and improves overall performance of GEMM computation.

Here we provide a simplified analysis on how dendritic architecture can help improve L2 cache hit rate therefore reduce communication cost on global memory access. Due to the complexity of the hierarchy cache mechanism, this analysis is not intended to be precise, but rather to help provide a theoretical understanding. 

Assume that we form a block group of $G$ rows of blocks from $\mA$ according to the standard approach to improve the efficiency of L2 cache~\cite{tritonlangMatrixMultiplication,CUTLASS}. For this to work, we need to fit the block of $G\cdot B_M$ rows and a single $B_N$ column into the L2 cache. We denote the capacity of the L2 cache as $Q$. That is, we have \[(G\cdot B_M+ B_N)\cdot L=Q.\] It is possible to put multiple columns from the $\mB$ matrix in the cache, but we can consider that it is absorbed in $B_N$. To utilize the cache efficiently, $G\cdot B_M$ need to stay in the cache while computations are performed along the $N$ axis~\cite{tritonlangMatrixMultiplication} where each step of computation requires read $B_N$ columns from the memory. In this way we can calculate that the memory read cost on matrix $\mB$ part is \[\frac{N\cdot L \cdot M}{B_M \cdot G}.\] And the read cost on $\mA$ part is $M\cdot L$, the total read cost would be \[\frac{N\cdot L \cdot M}{Q/L-B_N}+M\cdot L.\] It is desirable to set $B_N$ to a small value. Therefore the read cost will roughly be equal to \[\frac{N\cdot L^2 \cdot M}{Q}+M\cdot L.\] For models with dendritic neurons of $K$ dendrites, we will have a read cost of  \[\frac{N\cdot L^2 \cdot M}{(Q\sqrt{K})}+M\cdot L/\sqrt{K}=(\frac{N\cdot L^2 \cdot M}{Q}+M\cdot L)/\sqrt{K}.\] 
Therefore, we can significantly reduce global memory read access through adopting dendritic structure. 

\begin{algorithm}
\caption{Block-wise GEMM (Modified from the original algorithm from~\cite{tritonlangMatrixMultiplication})}
\label{alg:blockwise-gemm}
\begin{algorithmic}[1]
\State \textbf{Input:} Matrices $\mA \in \mathbb{R}^{M \times L}$, $\mB \in \mathbb{R}^{L \times N}$, $\mC \in \mathbb{R}^{M \times N}$
\State \textbf{Output:} Matrix $\mC$ containing the result of $\mC = \mA \times \mB$
\State \textbf{Define:} Block sizes $B_M$, $B_N$, $B_L$
\For{\textbf{each} $m$ in $0$ to $M$ by $B_M$} \Comment{Parallel execution over blocks of $C$}
    \For{\textbf{each} $n$ in $0$ to $N$ by $B_N$} \Comment{Parallel execution over blocks of $C$}
        \State Initialize $acc \gets \text{zeros}(B_M, B_N)$
        \For{\textbf{each} $l$ in $0$ to $L$ by $B_L$} \Comment{Iterate over blocks of $A$ and $B$}
            \State $a\_block \gets \mA[m : m+B_M, l : l+B_L]$
            \State $b\_block \gets \mB[l : l+B_L, n : n+B_N]$
            \State $acc \gets acc + (a\_block \times b\_block)$
        \EndFor
        \State $\mC[m : m+B_M, n : n+B_N] \gets acc$
    \EndFor
\EndFor
\end{algorithmic}
\end{algorithm}

\subsection{Empirical analysis}
The theoretical analysis presented above incorporates certain assumptions, such as a two-layer memory structure and explicit cache control, that do not fully align with the architecture of real-world hardware. To validate and extend this analysis, we conducted an empirical study of memory access costs during the inference process of typical neural network layers on an Nvidia A40 GPU. The results demonstrate that adopting a dendritic structure can significantly reduce communication costs, in alignment with the predictions of our theoretical framework. Due to the infeasibility of exploring the entire configuration space, this study does not aim to identify the optimal configuration for a specific setting. Instead, it aims to demonstrate clear advantages over well-established baselines that achieve the theoretical lower bound~\cite{tritonlangMatrixMultiplication,smith2018theory,olivry2022automatic}.

We measured the global memory access costs of dendritic neural network layers at three levels of computational complexity. Each baseline configuration used 
$K = 1$ (one dendrite per neuron), corresponding to a standard feedforward neural network with ReLU nonlinearity. The baseline implementation was built using the Triton standard framework~\cite{tritonlangMatrixMultiplication,CUTLASS}, designed to optimize memory access costs while achieving performance comparable to CuBLAS. In this scenario, computation is represented as $C = \sigma(A \times B)$, where 
$A \in \mathbb{R}^{M \times L}$, $B \in \mathbb{R}^{L \times N}$, and $C \in \mathbb{R}^{M \times N}$, with $\sigma$ denoting an element-wise nonlinearity. For all baseline experiments, $M = N = L$ was set to values from $\{1024, 4096, 16384\}$, representing the three complexity levels. 
We compared configurations with $K = 1, 4, 16, 64$ dendrites per neuron at each complexity level, ensuring that computational complexity remained consistent across baselines. For instance, at a complexity level of $M = 1024$ and $K = 4$, we set $M = 1024$, $L = 512$, and $N = 2048$. In this case, the final output matrix would have dimensions $\mathbb{R}^{1024 \times 512}$ after aggregating the outputs from every four dendrites.

We analyze global memory access patterns using Nvidia Nsight Compute and report the optimal results for each tested configuration. The best results were obtained by searching over $B_M, B_N, B_L \in {8, 16, 32, 64, 128}$ (values compatible with the Tensor Core architecture) and $G \in {1, 2, 4, 8, 16, 32, 64, 128, 256}$.

\begin{figure}[htb]
    \centering
    \includegraphics[width=0.75\linewidth]{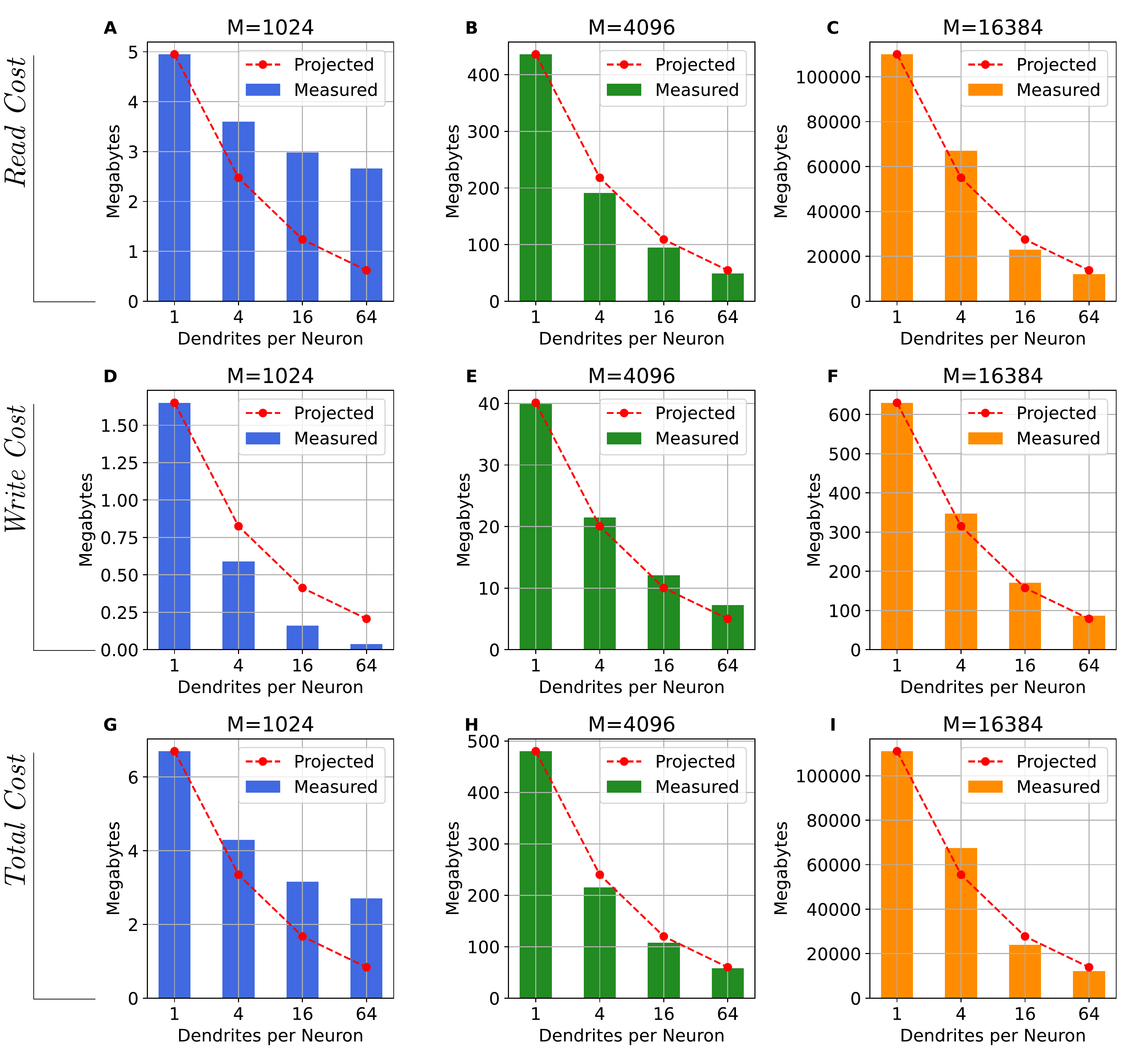}
    \caption{Communication cost analysis across neural network configurations.} \textbf{Top row:} Optimal read costs for networks with varying dendrites per neuron (K = 1, 4, 16, 64) at three complexity levels: M = 1024 (A), M = 4096 (B), and M = 16384 (C). Bar plots show measured costs, while dotted red lines indicate theoretical scaling projections. 
    \textbf{Middle row:} Optimal write costs for the same configurations, showing measured write costs and their theoretical scaling predictions.  \textbf{Bottom row: } Optimal total communication costs, combining both read and write operations for each configuration.     
    \label{fig:rw}
\end{figure}

Our findings indicate that the dendritic structure offers notable advantages. For smaller layers ($M = 1024$, Fig.\ref{fig:rw}A), the operator matrices can fit within the L2 cache, leading to lower performance gains from the dendritic structure compared to theoretical predictions (red dashed lines). However, as the matrix size surpasses the L2 cache capacity ($M = 4096$, Fig.\ref{fig:rw}B), memory access patterns begin to align with the expected $1/\sqrt{K}$ scaling. For reference, the A40 GPU used in these experiments has an L2 cache size of 6,144 KB, as documented in Nvidia Ampere GA102 GPU Architecture white paper. For the largest matrices ($M = 16384$, Fig. 1C), cache evictions lead to memory reads exceeding theoretical predictions. \\

Similar trends are observed for the write costs, as shown in the middle row of Fig. \ref{fig:rw}. Note that, in the case of small output matrices, the result matrices may remain in the L2 cache without being transferred to global memory. This can lead to observed write costs smaller than the size of the output matrix. To show a complete picture we also show the optimal total global memory access cost (read+write) measured across different settings shown at the bottom row of Fig. \ref{fig:rw}. 

The observations outlined above highlight a pathway to significantly reduce model inference energy costs, by reducing global memory access complexity. This approach could also enable GPU designs with lower memory bandwidth requirement. While reducing communication overhead can lower energy consumption, it does not always result in accelerated model inference due to potential bottlenecks in other parts of the inference pipeline. Notably, significant speedups in network layer computation are observed with larger matrix sizes, where memory I/O emerges as the primary bottleneck, as shown in Fig.~\ref{fig:speed}. As a side note, we also observe that configurations with a higher dendritic number per neuron (e.g., K=64) and low computational complexity (e.g., M=1024) tend to result in higher GPU inference time compared to simpler configurations (K=1, M=1024). This is likely due to the need for additional nonlinear activation computations, which must be processed by regular CUDA cores rather than the more efficient Tensor cores. Future architectural improvements 
or low-level code optimizations could address these inefficiencies.\\

\begin{figure}[htb]
    \centering
    \includegraphics[width=0.45\linewidth]{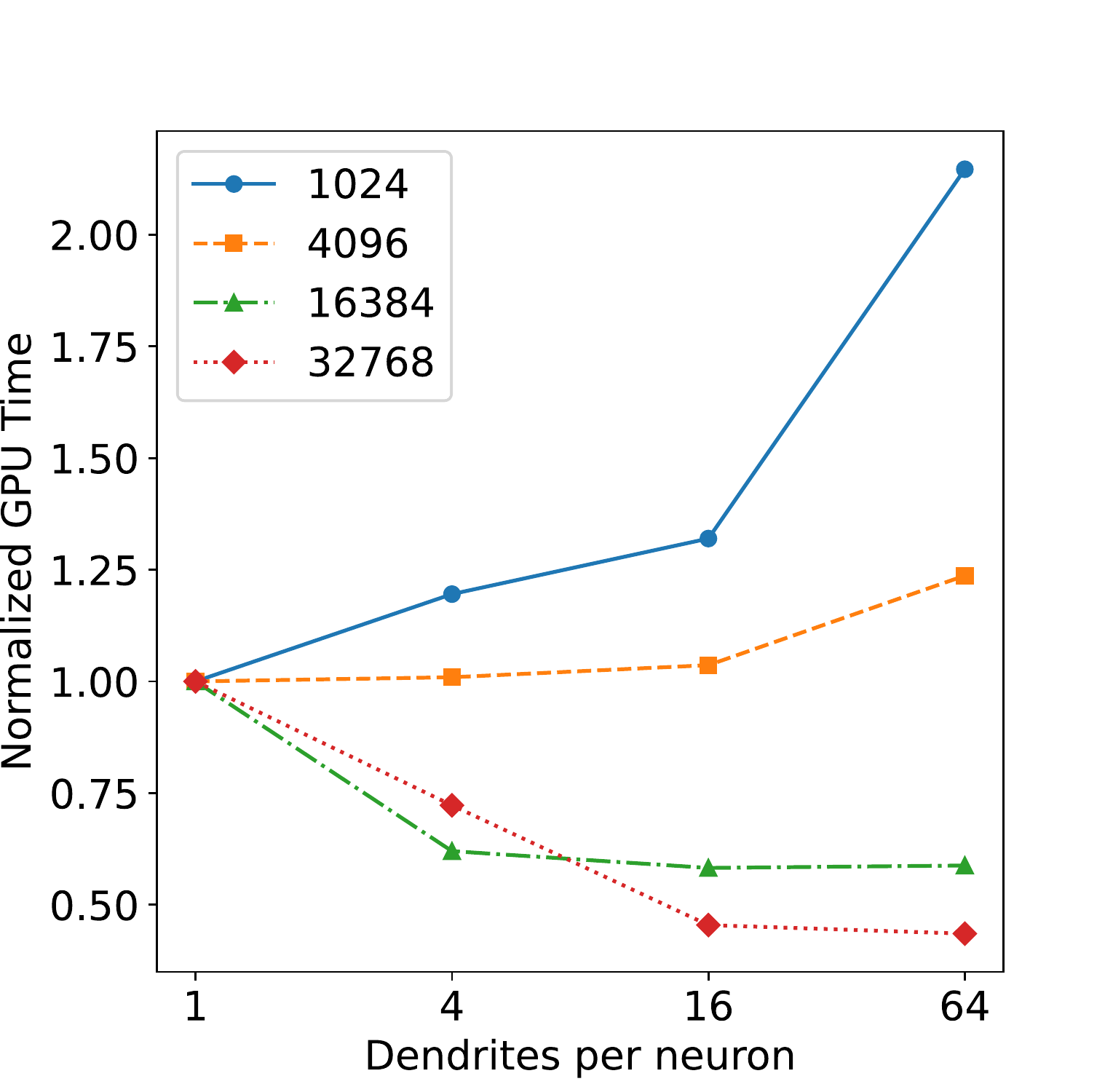}
    \caption{Normalized GPU runtime for different computational complexities, where \( M = N = L \) takes values from \( \{1024, 4096, 16384, 32768\} \), and the number of dendrites per neuron \( K \) varies (with \( K \in \{1, 4, 16, 64\} \)). The GPU runtime is normalized relative to the baseline case for each computational complexity level. }    
    \label{fig:speed}
\end{figure}

\section{Additional machine learning experimental analysis}\label{app:supp_ML}
\subsection{CIFAR-100 dataset with ResNet-18-style models}
We also applied our models to the CIFAR-100 dataset~\cite{krizhevsky2009learning}, which comprises 100 distinct object categories and is commonly employed in machine learning studies. Results are shown in Fig.~\ref{fig:cifar} and resemble our findings using the ImageNet dataset, where incorporating dendrites into a model with a fixed inter-layer communication budget consistently yields improved performance. Furthermore, dendritic-based models surpass point-neuron models with the same computational budget, provided that the inter-layer communication budget is above a certain threshold. As in the ImageNet dataset, we used the ResNet-18 model as the baseline architecture.

\begin{figure}[htb]
    \centering
    \includegraphics[width=0.99\linewidth]{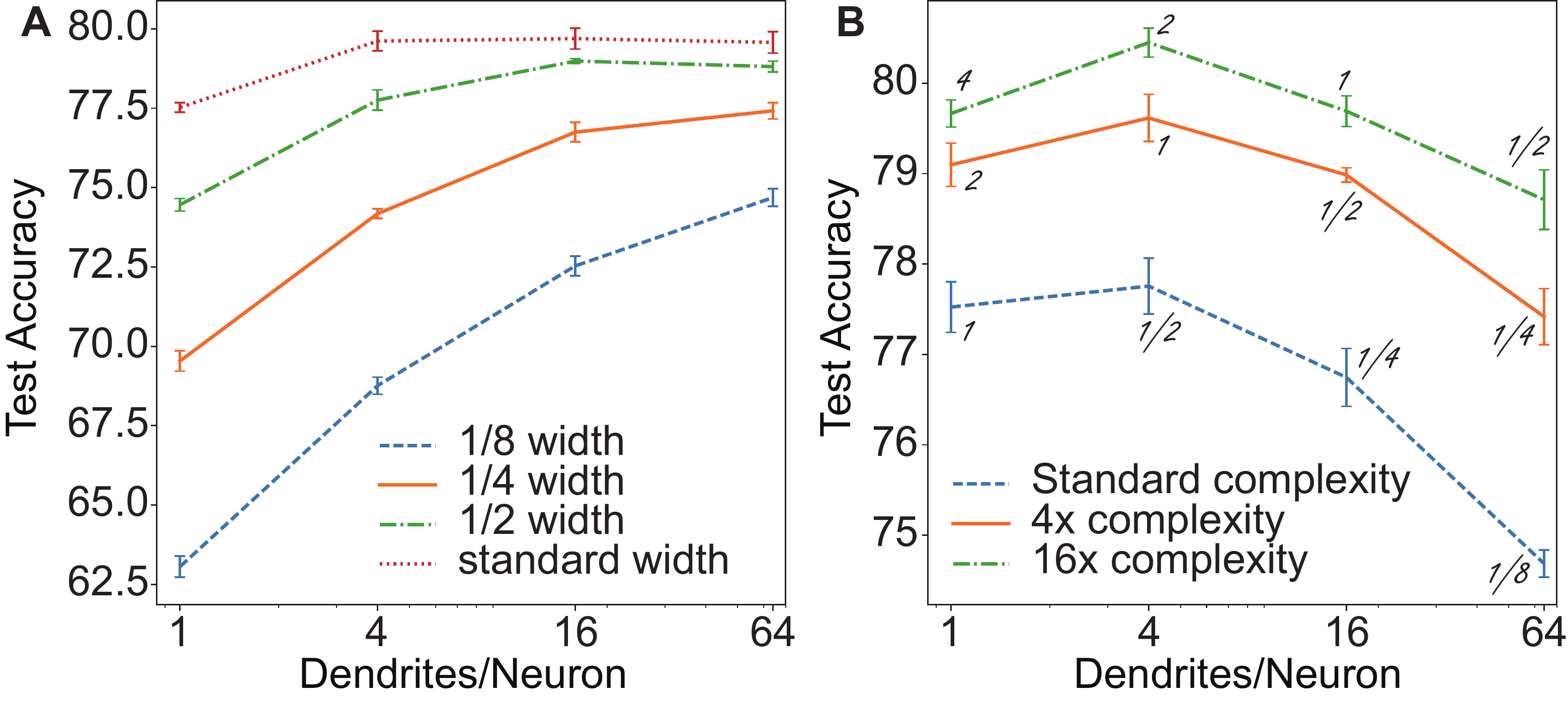}
    \centering
    \caption{Results on the CIFAR-100 dataset. Each experiment was performed $5$ times, with standard deviations displayed. (A) Test accuracy for models with varying numbers of dendrites per neuron at four distinct levels of network width. (B) Comparison of models with equivalent computational complexities at three different levels. The blue dashed curves represent the baseline ResNet-18 model and subsequent dendritic models with $K$ values of $4$, $16$, and $64$. The orange curve corresponds to models with twice the number of neurons (channels), and the green dashed curve represents models with four times the number of channels. The channel scale factors relative to the standard model are labeled on the curves in (B).
   }
    \label{fig:cifar}
\end{figure}
\subsection*{Additional results on Imagenet dataset experiment}
In the results section, we compare the performance of models when they are set to be of same computational complexity level. To obtain a full picture, we also compare models of same number of neurons $D$. The result is illustrated in Fig.~\ref{fig:AK}. We can observe
consistent performance improvement when more dendrites are added to the neurons.
\begin{figure}[htb!]
    \centering
    \includegraphics[width=0.95\linewidth]{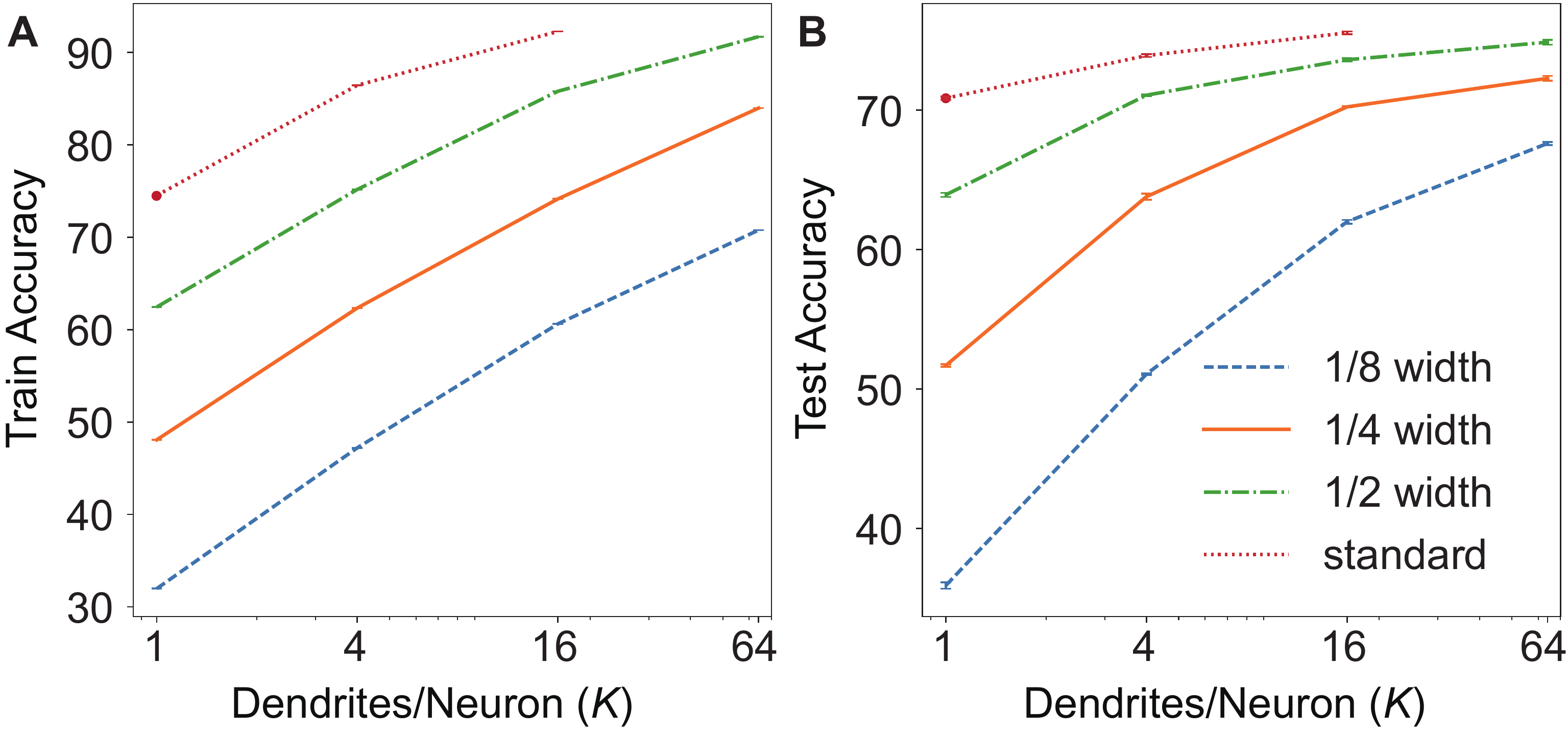}
    \centering
    \caption{Comparison of Resnet-18-style models composed of point and dendritic neurons trained on the ImageNet dataset. Each experiment was performed $5$ times, with standard deviations displayed. (A) Training accuracy, and (B) Test accuracy for models with varying numbers of dendrites per neuron at four distinct levels of network width. $x$-axis indicates the number of dendrites per neuron; models with one dendrite per neuron are point neuron-based models.
    \newline
   }
    \label{fig:AK}
\end{figure}

\subsection*{Non-Residual Convolutional Neural Network Performance on the ImageNet Dataset}
\label{app:non-res}

To ensure the robustness of our findings, we also used a convolutional neural network (CNN) model devoid of residual connections~\cite{he2016deep}. The base model for this experiment was a modified version of the original ResNet-18 network, from which we eliminated the residual connections. The original ResNet-18 model consists of four stages, each featuring two residual blocks. We removed one residual block from both the second and third stages to reduce computing costs. Fig.~\ref{fig:img_nores} illustrates the results from this modified, non-residual network, which are consistent with our original findings shown in Fig.~\ref{fig:imgnet}.

\begin{figure}[htb!]
    \centering
    \includegraphics[width=0.95\linewidth]{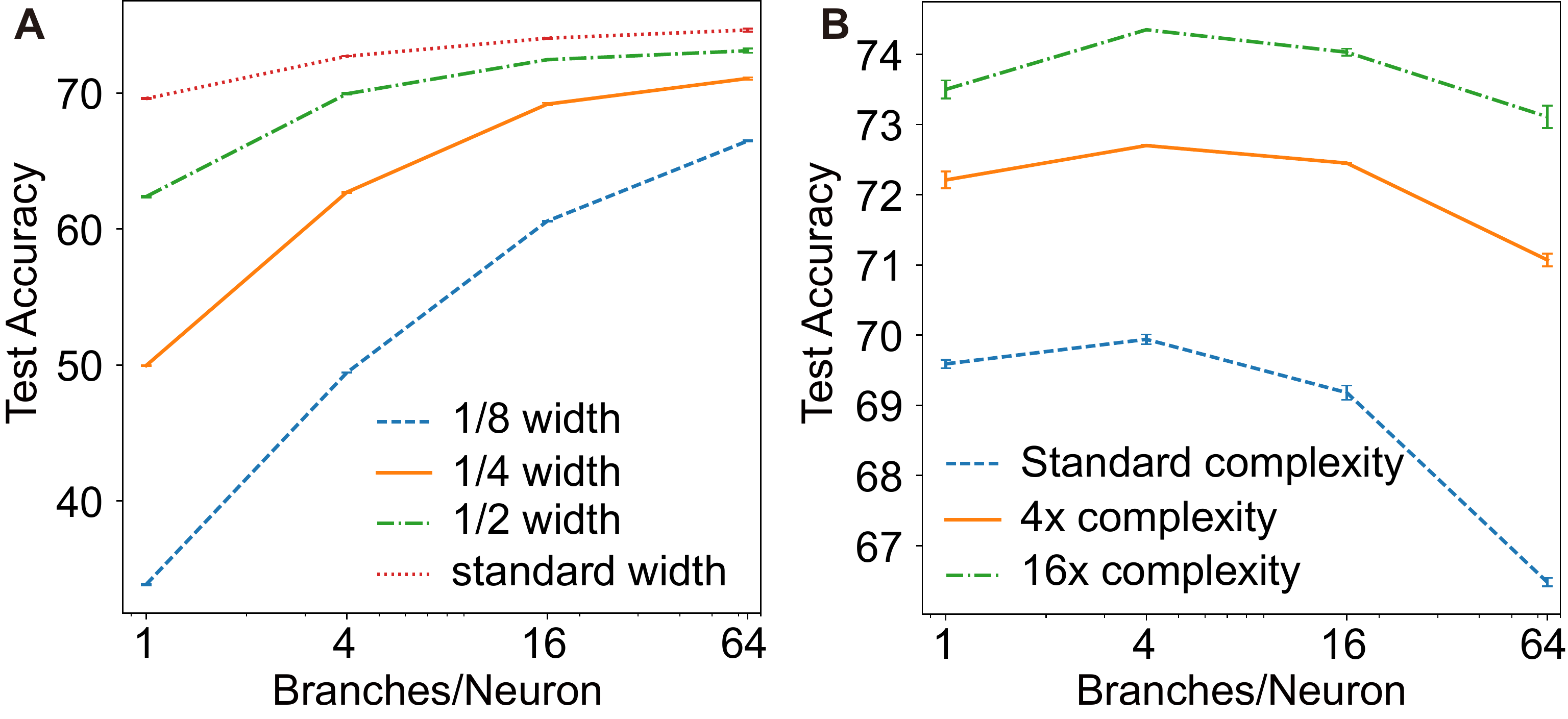}
    \centering
    \caption{Results on ImageNet dataset using neural network models without residual connections. Each experiment are performed $3$ times, with standard deviations displayed. (A) Test accuracy for models with varying numbers of dendrites per neuron at four distinct levels of network width. (B) Comparison of models with equivalent computational complexities at three different levels. The blue dashed curve represents the baseline and subsequent dendritic models with $K$ values of $4$, $16$, and $64$. The orange curve corresponds to models with twice the number of channels, and the green dashed curve represents four times the number of channels.
   }
    \label{fig:img_nores}
\end{figure}

\subsection*{Transformer model}
\label{app:deit}

This section investigates the impact of replacing the feedforward block within transformer-based neural networks. The specific feedforward block in question comprises a classic bottleneck architecture, as illustrated in Fig.~\ref{fig:bottle_neck}.

A bottleneck structure enhances the expressive capacity of a network module by expanding the number of channels in the middle layer. Conventionally, if the module input consists of $L$ channels, the middle layer is expanded to comprise $sL$ channels. Small integer values are commonly employed for $s$ in typical transformer-based models, with common choices including 2, 3, or 4. Subsequently, the module's output is reduced back to the original $L$ channels.

This bottleneck module confers greater expressivity power to the model than a standard two-layer network of $L$ channels while maintaining a modest input/output channel number for the module. This is similar to what dendritic structures try to achieve.

However, the bottleneck structure has an expanded middle layer, necessitating high communication bandwidth. Thus, the question arises: can a dendritic structure supplant the bottleneck structure while conferring additional benefits?

The naive substitution of a bottleneck structure with two dendritic layers is ineffective because the second layer comprises linear neurons. The pooling of linear neuron outputs does not confer inherent advantages to a nonlinear dendritic structure. Consequently, our design only employs a dendritic structure exclusively for the first layer of the block while retaining a linear layer for the second.

More precisely, for a bottleneck structure accepting an input dimension of $L$ and an expansion ratio of $s$, the corresponding first layer is assigned the dendritic branches equal to $2s-1$. This configuration maintains the input channel number for both layers at $L$, preserving the computational and parametric complexity at levels comparable to the original model.

An empirical examination involving a compact transformer model, as proposed by Hassani et al.~\cite{hassani2021escaping}, demonstrates that this modification incurs only a marginal performance decline. Specifically, test accuracy on the ImageNet dataset decreased from 80.9\% to 80.6\%, a negligible reduction considering the substantial decrease in peak activation output I/O within the block threefold less than before.

Considering the highly tuned nature of the transformer architecture, we posit that additional refinements to the model—particularly adjustments favoring the dendritic structure may unlock further potential for performance enhancement.

\subsection*{Speech recognition task}\label{app:speech}
In addition, we substantiate our theory with a speech recognition task. We employ models trained on the LibriSpeech dataset, which consists of approximately 1,000 hours of spoken English~\cite{panayotov2015librispeech2}. Owing to computing resource constraints, we utilize the train-clean-100 and train-clean-360 subsets for model training and the dev-clean subset for model evaluation.
The models used in this portion of the experiment are derived from the Jasper model~\cite{li2019jasper}, a 1D convolutional neural network. To lessen the computational burden during model training, we modified the original model by eliminating the dense residual connections and significantly reducing the number of blocks in the model to arrive at a baseline point neuron based model. Further details regarding the modifications to the models can be found in the accompanying code.

For this part, we carry out two distinct sets of experiments. The first set focuses on models of equivalent computational complexity, and the second emphasizes models sharing the same inter-layer communication cost. 

In the first set of experiments, we evaluated models of two distinct computational complexity levels, varying the neuron configurations. Specifically, the configurations encompassed point neurons and dendritic neurons with varying numbers of dendrites. The results for this segment of experiments are displayed in Table \ref{tab:speech1}. Analogous to previous experiments, we observed that models utilizing dendritic neurons were able to achieve comparable performance relative to the point neuron-based models with equivalent computational complexity if they are equipped with efficient inter-layer communication bandwidth.

The second set of experiments is conducted employing models that retain the same inter-layer communication cost. Our experimental procedure begins with a point neuron-based model, which possesses one-fourth of the inter-layer communication complexity compared to the baseline model. This point neuron model is subsequently replaced with dendritic neuron models that contain 4 and 16 dendrites respectively. The corresponding results are systematically presented in Table~\ref{tab:speech2}. Upon analyzing these results, it becomes apparent that the performance of the model progressively enhances as we incorporate neurons with an increased number of dendrites.

\begin{figure}[htb!]
    \centering
    \includegraphics[width=0.9\linewidth]{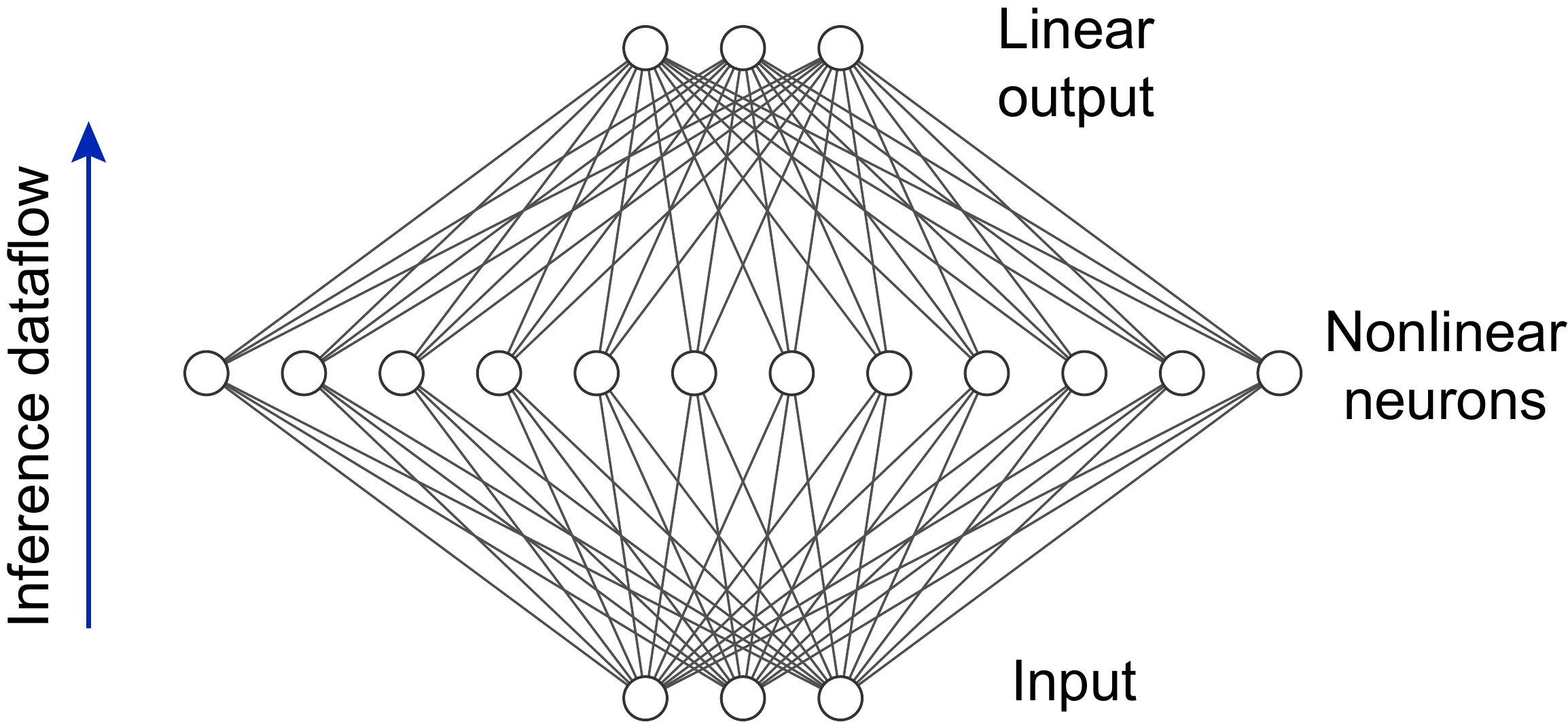}
    \caption{Schematic representation of a bottleneck neural network module comprising two interconnected layers.}
    \label{fig:bottle_neck}
\end{figure}

\begin{table}[htb!]
\centering
\caption{Comparison of the performance of dendritic models with varying numbers of dendrites per neuron on the LibriSpeech dataset. The table presents models with two levels of computational complexity. To maintain equivalent computational complexity when increasing the number of dendrites in a neuron, the number of inter-layer channels is proportionally reduced, as indicated in the table.}
\label{tab:speech1}
\begin{tabular}{|ccc|}
\hline
\multicolumn{1}{|c|}{\textbf{\# of Dendrites}} & \multicolumn{1}{c|}{\textbf{Channel scaling factor}} & \textbf{Test error}       \\ \hline
\multicolumn{3}{|c|}{\textit{1 st complexity level(baseline)}}                                                                    \\ \hline
\multicolumn{1}{|c|}{1 (baseline)}             & \multicolumn{1}{c|}{1}                               & 7.72                      \\ \hline
\multicolumn{1}{|c|}{4}                        & \multicolumn{1}{c|}{1/2}                             & 7.92                      \\ \hline
\multicolumn{1}{|c|}{16}                       & \multicolumn{1}{c|}{1/4}                             & 8.28                      \\ \hline
\multicolumn{3}{|c|}{\textit{2nd Complexity level}}                                                                               \\ \hline
\multicolumn{1}{|c|}{1}                        & \multicolumn{1}{c|}{2}                               & \multicolumn{1}{l|}{6.69} \\ \hline
\multicolumn{1}{|c|}{4}                        & \multicolumn{1}{c|}{1}                               & \multicolumn{1}{l|}{6.69} \\ \hline
\multicolumn{1}{|c|}{16}                       & \multicolumn{1}{c|}{1/2}                             & \multicolumn{1}{l|}{6.89} \\ \hline
\end{tabular}

\end{table}

\begin{table}[htb!]
\centering
\caption{Performance Evaluation of Dendritic Models with varying dendritic counts per neuron evaluated on the LibriSpeech Dataset. The models in this comparison have the same inter-layer communication cost.}
\label{tab:speech2}
\begin{tabular}{|c|c|c|}
\hline
\# of Dendrites & Channel scaling factor &  Test error\\ \hline
1               & 1/4       &     15.39                  \\ \hline
4               & 1/4       &    10.49                 \\ \hline
16              & 1/4       &    8.23                 \\ \hline
\end{tabular}

\end{table}

\end{appendices}

\end{document}